%% file: main_final.tex
\documentclass[10pt,twocolumn,letterpaper]{article}

\usepackage{cvpr}               % To produce the CAMERA-READY version

\usepackage[accsupp]{axessibility}
\usepackage{capt-of}
\usepackage{xspace}
\usepackage{xcolor}
\usepackage{enumitem}
\usepackage{amsmath,amsthm,amssymb,amsfonts,dsfont,pifont,bm,bbm,mathrsfs,mathtools,nicefrac,extarrows,relsize}
\usepackage{algorithm,algpseudocode,listings}
\usepackage{booktabs,multirow,adjustbox,diagbox,threeparttable,tabularray,setspace}

\definecolor{cvprblue}{rgb}{0.21,0.49,0.74}
\usepackage[pagebackref,breaklinks,colorlinks,citecolor=cvprblue,bookmarks=false]{hyperref}
\usepackage{wrapfig}
\usepackage[capitalize]{cleveref}  % Should be loaded after 'hyperref', and works perfectly with 'subfigure'.
\newtheorem{theorem}{Theorem}

\newtheorem{lemma}{Lemma}

\newtheorem{remark}{Remark}
\crefname{section}{Sec.}{Secs.}
\Crefname{section}{Section}{Sections}
\crefname{appendix}{App.}{Apps.}
\Crefname{appendix}{Appendix}{Appendices}
\crefname{table}{Tab.}{Tabs.}
\Crefname{table}{Table}{Tables}
\crefname{figure}{Fig.}{Figs.}
\Crefname{figure}{Figure}{Figures}
\crefname{equation}{Eq.}{Eqs.}
\Crefname{equation}{Equation}{Equations}
\crefname{theorem}{Thm.}{Thms.}
\Crefname{theorem}{Theorem}{Theorems}
\crefname{lemma}{Lem.}{Lems.}
\Crefname{lemma}{Lemma}{Lemmas}
\crefname{remark}{Rem.}{Rems.}
\Crefname{remark}{Remark}{Remarks}
\crefname{corollary}{Cor.}{Cors.}
\Crefname{corollary}{Corollary}{Corollaries}
\crefname{algorithm}{Alg.}{Algs.}
\Crefname{algorithm}{Algorithm}{Algorithms}
\definecolor{cellred}{RGB}{213, 123, 101}
\definecolor{cellgreen}{RGB}{0, 205, 0}
\definecolor{cellblue}{RGB}{54, 125, 189}
\definecolor{codegreen}{rgb}{0,0.6,0}
\definecolor{codegray}{rgb}{0.5,0.5,0.5}
\definecolor{codepurple}{rgb}{0.58,0,0.82}
\definecolor{backcolour}{rgb}{1.0,1.0,1.0}
\lstdefinestyle{mystyle}{
    backgroundcolor=\color{backcolour},
    commentstyle=\color{codegreen},
    keywordstyle=\color{magenta},
    numberstyle=\tiny\color{codegray},
    stringstyle=\color{codepurple},
    basicstyle=\ttfamily\scriptsize,
    breakatwhitespace=false,
    breaklines=true,
    captionpos=b,
    keepspaces=true,
    numbers=left,
    numbersep=5pt,
    showspaces=false,
    showstringspaces=false,
    showtabs=false,
    tabsize=2
}
\newcommand{\tocite}[1]{{\color{red} [TO CITE]}}
\newcommand{\methodname}{ReCFG}
\newcommand{\method}{\mbox{\texttt{\methodname}}\xspace}

\def\paperID{5384}
\def\confName{CVPR}
\def\confYear{2025}
\title{Rectified Diffusion Guidance for Conditional Generation}

\author{%
    Mengfei Xia$^{1,2}$\thanks{Work finished during internship at Ant Group.} \quad
    Nan Xue$^2$ \quad
    Yujun Shen$^2$\footnotemark[2] \quad
    Ran Yi$^3$ \quad
    Tieliang Gong$^4$ \quad
    Yong-Jin Liu$^1$\thanks{Corresponding author.} \quad \\[5pt]
    $^1$Tsinghua University \quad
    $^2$Ant Group \quad \\[2pt]
    $^3$Shanghai Jiao Tong University \quad
    $^4$Xi'an Jiao Tong University
}

\begin{document}

\maketitle

\input{sections_final/0_abstract.tex}
\input{sections_final/1_intro.tex}

\input{sections_final/2_related.tex}

\input{sections_final/3_method.tex}

\input{sections_final/4_exp.tex}

\input{sections_final/5_conclusion.tex}

\input{sections_final/8_acknowledgement.tex}

\input{sections_final/6_ref.tex}
\input{sections_final/7_appendix.tex}

\end{document}

%% file: sections_final/0_abstract.tex
\begin{abstract}

Classifier-Free Guidance (CFG), which combines the conditional and unconditional score functions with two coefficients summing to one, serves as a practical technique for diffusion model sampling.
Theoretically, however, denoising with CFG \textit{cannot} be expressed as a reciprocal diffusion process, which may consequently leave some hidden risks during use.
In this work, we revisit the theory behind CFG and rigorously confirm that the improper configuration of the combination coefficients (\textit{i.e.}, the widely used summing-to-one version) brings about expectation shift of the generative distribution.
To rectify this issue, we propose \method%
\footnote{ReCFG, pronounced as ``reconfigure'', is the abbreviation for ``rectified Classifier-Free Guidance''.}
with a relaxation on the guidance coefficients such that denoising with \method strictly aligns with the diffusion theory.
We further show that our approach enjoys a \textbf{\textit{closed-form}} solution given the guidance strength.
That way, the rectified coefficients can be readily pre-computed via traversing the observed data, leaving the sampling speed barely affected.
Empirical evidence on real-world data demonstrate the compatibility of our post-hoc design with existing state-of-the-art diffusion models, including both class-conditioned ones (\textit{e.g.}, EDM2 on ImageNet) and text-conditioned ones (\textit{e.g.}, SD3 on CC12M), without any retraining.
Code is available at \href{https://github.com/thuxmf/recfg}{https://github.com/thuxmf/recfg}.

\end{abstract}

%% file: sections_final/1_intro.tex
\section{Introduction}\label{sec:intro}

Diffusion probabilistic models (DPMs)~\citep{sohl2015deep,ho2020denoising,song2020score}, known simply as {diffusion models}, have achieved unprecedented capability improvement of high-resolution image generation.
It is well recognized that, DPMs are the most prominent generative paradigm for a broad distribution (\textit{i.e.}, text-to-image generation)~\citep{podell2024sdxl,chen2024pixartalpha,esser2024sd3}.
Among DPM literature, Classifier-Free Guidance (CFG)~\citep{ho2022classifierfree} serves as an essential factor, enabling better conditional sampling in various fields~\citep{rombach2022high,poole2023dreamfusion}.
Vanilla conditional sampling via DPMs introduces the conditional score function $s_t(\mathbf x, c)=\nabla_{\mathbf x_t}\log q_t(\mathbf x_t|c)$, resulting in poor performance in which synthesized samples appear to be visually incoherent and not faithful to the condition, even for large-scale models~\citep{rombach2022high}.
By drawing lessons from Bayesian theory, CFG employs an interpolation between conditional and unconditional score functions with a preset weight $\gamma$, \textit{i.e.},
\begin{align}
s_{t,\gamma}(\mathbf x,c)=\gamma\nabla_{\mathbf x}\log q_t(\mathbf x|c)+(1-\gamma)\nabla_{\mathbf x}\log q_t(\mathbf x),
\end{align}
in which $\nabla_{\mathbf x_t}\log q_t(\mathbf x_t)$ is the unconditional score function by annihilating the condition effect.
By doing so, DPMs turn out to formulate the underlying distribution with a gamma-powered distribution~\citep{bradley2024classifierfreeguidancepredictorcorrector}, \textit{i.e.},
\begin{align}
q_{t,\gamma}(\mathbf x|c)=q_t(\mathbf x|c)^\gamma q_t(\mathbf x)^{1-\gamma},
\end{align}
which is proportional to $q_t(\mathbf x)q_t(c|\mathbf x)^\gamma$.
Enlarging $\gamma>1$ focuses more on the classifier effect $q_t(c|\mathbf x)$, concentrating on better exemplars of given condition and thereby sharpening the gamma-powered distribution.
In other words, CFG is designed to promote the influence of the condition.

However, inspired by seminal works~\citep{bradley2024classifierfreeguidancepredictorcorrector}, we argue that denoising with CFG cannot be expressed as a reciprocal of vanilla diffusion process by adding Gaussian noises, since the normally nonzero score function expectation of gamma-powered $q_{t,\gamma}(\mathbf x|c)$ violates the underlying theory of DPMs.
Theoretically, score functions with zero expectation at all timesteps guarantee that the denoised $\tilde{\mathbf x}_0$ has expectation $\mathbb E[\tilde{\mathbf x}_0]=\frac{\alpha_0}{\alpha_T}\mathbb E[\mathbf x_T]$, thus $\mathbb E[\tilde{\mathbf x}_0]=\mathbb E[\mathbf x_0]$ and no bias on the conditional fidelity.
Therefore, this theoretical flaw leaves some hidden risks during use, manifesting as a severe expectation shift phenomenon, \textit{i.e.}, the expectation of the gamma-powered distribution will be shifted away from the ground-truth of the conditional distribution $q_t(\mathbf x|c)$.
This is more conspicuous when applying larger $\gamma$.
\cref{fig:toy_model} clearly clarifies the expectation shift, in which the peak of induced distribution via CFG in \textcolor{cellred}{red} fails to coincide with that of ground-truth $q_0(\mathbf x_0|c)$.
This theoretical flaw is known in theory~\citep{du2024reducereuserecyclecompositional,karras2024guidingdiffusionmodelbad,bradley2024classifierfreeguidancepredictorcorrector}, while being largely ignored in practice.

\input{figs_final/toy_model/fig.tex}

In this work, we first revisit the formulation of native CFG, theoretically confirming its flaw that we concluded above and summarizing as \Cref{thm:cfg}.
Then, to quantitatively reveal the consequent expectation shift phenomenon by CFG, we employ a toy distribution, enjoying closed-form description of the behavior on the gamma-powered distribution.
Under the toy settings, we analytically calculate the function of the precise value of expectation shift in correspondence with $\gamma$, as summarized in \Cref{thm:gs_shift}.
Motivated by theoretical compatibility and canceling the expectation shift, we apply relaxation on the guidance coefficients in native CFG by circumventing the constraint that two coefficients sum to one, enabling a more flexible control on the induced distributions.
To be more concrete, we propose to formulate the underlying distribution with \textit{two} coefficients, \textit{i.e.},
\begin{align}
q_{t,\gamma_1,\gamma_0}(\mathbf x|c)=q_t(\mathbf x|c)^{\gamma_1} q_t(\mathbf x)^{\gamma_0}.
\end{align}
Aiming at consistency with the diffusion theory and thus better guidance efficacy, we specially design the constraints on $\gamma_1$ and $\gamma_0$, and theoretically confirm the feasibility.
We further provide a closed-form solution to the constraints, and propose an algorithm to analytically determine $\gamma_0$ from a pre-computed lookup table in a post-hoc fashion.
Thanks to the neat formulation, we can employ pixel-wise $\gamma_0$ according to the lookup table involving guidance strength $\gamma_1$, condition $c$ and timestep $t$, as demonstrated in \cref{fig:coeff}.
We name the above process \method.
Compared with a global CFG weight applied on all denoising steps and pixels, \method may achieve more flexible and accurate guidance.
Experiments with state-of-the-art DPMs, including both class-conditioned ones (\textit{e.g.}, EDM2~\citep{Karras2024edm2}) and text-conditioned ones (\textit{e.g.}, SD3~\citep{esser2024sd3}) under different NFEs and guidance strengths show that our \method can achieve better guidance efficacy without retraining or extra time cost during inference stage.
Hence, our work offers a new perspective on guided sampling of DPMs, encouraging more studies in the field of guided generation.

%% file: figs_final/toy_model/fig.tex
\begin{figure*}[t]
\centering
\includegraphics[width=0.8\textwidth]{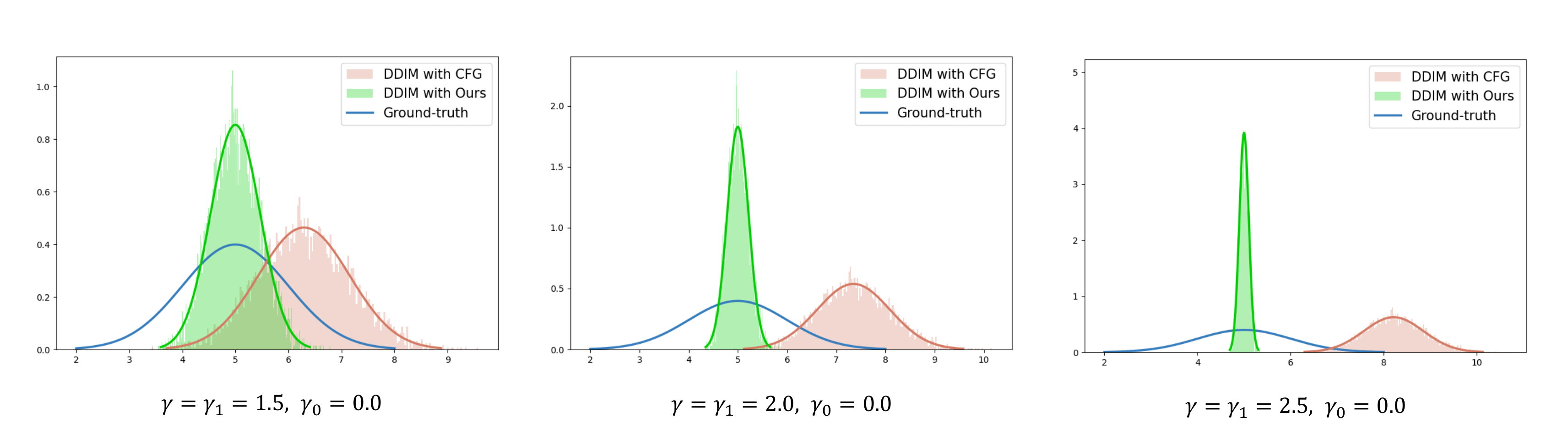}
\vspace{-10pt}
\caption{
    \textbf{Visualization} of expectation shift.
    The demonstrated toy data is simulated by $q_0(\mathbf x_0|c)\sim\mathcal N(c,1)$, $q(c)\sim\mathcal N(0,1)$, $q_0(\mathbf x_0)\sim\mathcal N(0,2)$.
    Gamma-powered distribution $q_{0,\gamma}(\mathbf x_0|c)$ from CFG~\citep{ho2022classifierfree} fails to recover the same conditional expectation as ground-truth due to expectation shift (\textit{i.e.}, probability density function and histogram by DDIM~\citep{song2020denoising} sampler in \textbf{\textcolor{cellred}{red}}).
    To make a further step, larger $\gamma$ suggests more severe expectation shift, \textit{i.e.}, the peak of $q_{0,\gamma}(\mathbf x_0|c)$ tends further away from $q_0(\mathbf x_0|c)$ (\textit{i.e.}, probability density function in \textbf{\textcolor{cellblue}{blue}}) as $\gamma$ goes from $1.5$ to $2.5$.
    As a comparison, our \method successfully recovers the ground-truth expectation and smaller variance (\textit{i.e.}, probability density function and histogram by DDIM~\citep{song2020denoising} sampler in \textbf{\textcolor{cellgreen}{green}}), consistent with the motivation of guided sampling.
}
\label{fig:toy_model}
\vspace{-10pt}
\end{figure*}

%% file: sections_final/2_related.tex
\section{Related Work}\label{sec:related}

\noindent\textbf{DPMs and conditional generation.}
Diffusion probabilistic model (DPM) introduces a new scheme of generative modeling, formulated by forward diffusing and reverse denoising processes in a differential equation fashion~\citep{sohl2015deep,ho2020denoising,song2020score}.
Practically, it is trained by optimizing the variational lower bound.
Benefiting from this breakthrough, DPM achieves high generation fidelity, and even beat GANs on image generation.
By drawing lessons from conditional distribution, conditional generation~\citep{choi2021ilvr,lhhuang2023composer} takes better advantage of intrinsic intricate knowledge of data distribution, making DPM easier to scale up and the most promising option for generative modeling.
Among the literature, text-to-image generation injects the embedding of text prompts to DPM, faithfully demonstrating the text content~\citep{podell2024sdxl,chen2024pixartalpha,esser2024sd3}.

\noindent\textbf{Classifier-Free Guidance.}
Classifier-Free Guidance (CFG) serves as the successor of Classifier Guidance (CG)~\citep{dhariwal2021diffusion}, circumventing the usage of a classifier for noisy images.
Both CFG and CG are based on Bayesian theory, and attempt to formulate the underlying distribution by concentrating more on condition influence, achieving better conditional fidelity.
Despite great success in large-scale conditional generation, CFG faces a technical flaw that the guided distribution is not theoretically guaranteed to recover the ground-truth conditional distribution~\citep{du2024reducereuserecyclecompositional,karras2024guidingdiffusionmodelbad,bradley2024classifierfreeguidancepredictorcorrector,chidambaram2024what}.
To be more detailed, there exists a shifting issue that the expectation of guided distribution is drifted away from the correct one~\citep{bradley2024classifierfreeguidancepredictorcorrector,chidambaram2024what}.
This phenomenon may harm the condition faithfulness, especially for extremely broad distribution (\textit{e.g.}, open-vocabulary synthesis).

%% file: sections_final/3_method.tex
\section{Method}\label{sec:method}

\subsection{Background on conditional DPMs and CFG}\label{subsec:method.1}

Let $\mathbf x_0\in\mathbb R^D$ be a $D$-dimensional random variable with an unknown distribution $q_0(\mathbf x_0|c)$, where $c\sim q(c)$ is the given condition.
DPM~\citep{sohl2015deep,song2020score,ho2020denoising} introduces a forward process $\{\mathbf x_t\}_{t\in(0,T]}$ by gradually corrupting data signal of $\mathbf x_0$ with Gaussian noise, \textit{i.e.}, the following transition distribution holds for any $t\in(0,T]$:
\begin{align}\label{eq:tran}
q_{0t}(\mathbf x_t|\mathbf x_0,c)=q_{0t}(\mathbf x_t|\mathbf x_0)=\mathcal N(\alpha_t\mathbf x_0,\sigma_t^2\mathbf I),
\end{align}
in which $\alpha_t,\sigma_t\in\mathbb R^+$ are differentiable functions of $t$ with bounded derivatives, referred to as the \textit{noise schedule}.
Let $q_t(\mathbf x_t|c)$ be the marginal distribution of $\mathbf x_t$ conditioned on $c$, DPM ensures that $q_T(\mathbf x_T|c)\approx\mathcal N(\mathbf 0,\sigma^2\mathbf I)$ for some $\sigma>0$, and the signal-to-noise-ratio (SNR) $\alpha_t^2/\sigma_t^2$ is strictly decreasing with respect to timestep $t$~\citep{kingma2021variational}.

Seminal works~\citep{kingma2021variational,song2020score} studied the underlying stochastic differential equation (SDE) and ordinary differential equation (ODE) theory of DPM.
The forward and reverse processes are as below for any $t\in[0,T]$:
\begin{align}
\mathrm d\mathbf x_t&=f_t\mathbf x_t\mathrm dt+g_t\mathrm d\mathbf w_t,\quad\mathbf x_0\sim q_0(\mathbf x_0|c),\label{eq:forward} \\
\mathrm d\mathbf x_t&=[f_t\mathbf x_t-g^2_t\nabla_{\mathbf x_t}\log q_t(\mathbf x_t|c)]\mathrm dt+g_t\mathrm d\bar{\mathbf w}_t,\label{eq:reverse}
\end{align}
where $\mathbf w_t,\bar{\mathbf w}_t$ are standard Wiener processes in forward and reverse time, respectively, and $f_t,g_t$ have closed-form expressions with respect to $\alpha_t,\sigma_t$.
The unknown $\nabla_{\mathbf x_t}\log q_t(\mathbf x_t|c)$ is referred to as the conditional score function.
Probability flow ODE (PF-ODE) from Fokker-Planck equation enjoys the identical marginal distribution at each $t$ as that of the SDE in \cref{eq:reverse}, \textit{i.e.},
\begin{align}\label{eq:probODE}
\frac{\mathrm d\mathbf x_t}{\mathrm dt}=f_t\mathbf x_t-\frac{1}{2}g^2_t\nabla_{\mathbf x_t}\log q_t(\mathbf x_t|c).
\end{align}

Technically, DPM implements sampling by solving the reverse SDE or ODE from $T$ to 0.
To this end, it introduces a neural network $\boldsymbol \epsilon_\theta(\mathbf x_t,c,t)$, namely the noise prediction model, to approximate the conditional score function from the given $\mathbf x_t$ and $c$ at timestep $t$, \textit{i.e.}, $\boldsymbol\epsilon_\theta(\mathbf x_t,c,t)=-\sigma_t\nabla_{\mathbf x_t}\log q_t(\mathbf x_t|c)$, where the parameter $\theta$ can be optimized by the objective below:
\begin{align}\label{eq:dpm_loss}
\mathbb E_{\mathbf x_0,\boldsymbol\epsilon,c,t}[\omega_t\|\boldsymbol\epsilon_\theta(\mathbf x_t,c,t) - \boldsymbol\epsilon\|_2^2],
\end{align}
where $\omega_t$ is the weighting function, $\boldsymbol\epsilon\sim\mathcal N(\mathbf 0,\mathbf I)$, $c\sim q(c)$, $\mathbf x_t=\alpha_t\mathbf x_0+\sigma_t\boldsymbol\epsilon$, and $t\sim\mathcal U[0,T]$.

For better condition fidelity, during denoising stage, CFG~\citep{ho2022classifierfree} turns to use a linear interpolation between conditional and unconditional score functions, \textit{i.e.},
\begin{align}\label{eq:cfg}
s_{t,\gamma}(\mathbf x,c)=\gamma\nabla_{\mathbf x}\log q_t(\mathbf x|c)+(1-\gamma)\nabla_{\mathbf x}\log q_t(\mathbf x).
\end{align}
Then PF-ODE can be rewritten as
\begin{align}\label{eq:pf_ode}
\frac{\mathrm d\mathbf x_t}{\mathrm dt}=f_t\mathbf x_t-\frac{1}{2}g^2_ts_{t,\gamma}(\mathbf x_t,c).
\end{align}

We further describe the CFG under the original DDIM theory.
Recall that DDIM turns out to formulate discrete non-Markovian forward diffusing process such that the reverse denoising process obeys the distribution with parameters $\{\delta_t\}_{t=0}^T$~\citep{song2020denoising}:
\vspace{-3pt}
\begin{align}
&\;q_\delta(\mathbf x_{t-1}|\mathbf x_t,\mathbf x_0,c)=q_\delta(\mathbf x_{t-1}|\mathbf x_t,\mathbf x_0) \\
\sim&\;\mathcal N\left(\alpha_{t-1}\mathbf x_0+\sqrt{\sigma_{t-1}^2-\delta_t^2}\cdot\frac{\mathbf x_t-\alpha_t\mathbf x_0}{\sigma_t},\delta_t^2\mathbf I\right).
\end{align}
Trainable generative process $p_\theta(\mathbf x_{t-1}|\mathbf x_t,c)$ is designed to leverage $q_\delta(\mathbf x_{t-1}|\mathbf x_t,\mathbf x_0,c)$ with a further designed denoised observation $\mathbf f_\theta^t$ with noise prediction model $\boldsymbol\epsilon_\theta$, \textit{i.e.},
\vspace{-12pt}
\begin{align}
\mathbf f_\theta^t(\mathbf x_t,c)&=\frac{1}{\alpha_t}(\mathbf x_t-\sigma_t\boldsymbol\epsilon_\theta(\mathbf x_t,c,t)), \\
\hspace{-2mm} p_\theta(\mathbf x_{t-1}|\mathbf x_t,c)&=
\begin{cases}
\mathlarger{q_\delta(\mathbf x_{t-1}|\mathbf x_t,\mathbf f_\theta^t(\mathbf x_t,c),c)}, & t>1, \\[2pt]
\mathlarger{\mathcal N(\mathbf f_\theta^t(\mathbf x_1),\sigma_1^2\mathbf I)}, & t=1.
\end{cases}
\end{align}

DDIM proves that for any $\{\delta_t\}_t$, score matching of non-Markovian process above is equivalent to native DPM.
With CFG weight $\gamma$, we generalize the theory as below:
\begin{align}
\hat{\boldsymbol\epsilon}_\theta(\mathbf x_t,c,t)&=\gamma\boldsymbol\epsilon_\theta(\mathbf x_t,c,t)+(1-\gamma)\boldsymbol\epsilon_\theta(\mathbf x_t,t), \\
\hat{\mathbf f}_{\theta}^t(\mathbf x_t,c)&=\frac{1}{\alpha_t}(\mathbf x_t-\sigma_t\hat{\boldsymbol\epsilon}_\theta(\mathbf x_t,c,t)),\label{eq:cfg_ddim1} \\
\hspace{-2mm} \hat p_{\theta}(\mathbf x_{t-1}|\mathbf x_t,c)&=
\begin{cases}
\mathlarger{q_\delta(\mathbf x_{t-1}|\mathbf x_t,\hat{\mathbf f}_{\theta}^t(\mathbf x_t,c),c)}, & t>1, \\[2pt]
\mathlarger{\mathcal N(\hat{\mathbf f}_{\theta}^t(\mathbf x_1,c),\sigma_1^2\mathbf I)}, & t=1.
\end{cases}\label{eq:cfg_ddim2}
\end{align}
Native DDIM theory still holds since $q_\delta(\mathbf x_{t-1}|\mathbf x_t,\mathbf x_0,c)=q_\delta(\mathbf x_{t-1}|\mathbf x_t,\mathbf x_0)$, \textit{i.e.}, with the definition
\begin{align}
J_{\delta,\gamma}(\boldsymbol\epsilon_\theta)=\mathbb E_{q_\delta(\mathbf x_{0:T}|c)}\left[\log\frac{q_\delta(\mathbf x_{1:T}|\mathbf x_0,c)}{\hat p_{\theta}(\mathbf x_{0:T}|c)}\right],
\end{align}
we have the following theorem.
Proof is in \Cref{subsec:proof.2}.

\begin{theorem}\label{thm:cfg}
For any $\{\delta_t\}_t$ and $\gamma>1$, $J_{\delta,\gamma}$ is equivalent to native DPM under CFG up to a constant.
However, denoising with CFG is not a reciprocal of the original diffusion process with Gaussian noise due to nonzero expectation of unconditional score function $\mathbb E_{q_t(\mathbf x_t|c)}[\nabla_{\mathbf x_t}\log q_t(\mathbf x_t)]$.
\end{theorem}

\begin{remark}\label{rem:cfg}
$\boldsymbol\epsilon_\theta(\mathbf x_t,c,t)$ and $\boldsymbol\epsilon_\theta(\mathbf x_t,t)$ are proportional to $\nabla_{\mathbf x_t}\log q_t(\mathbf x_t|c)$ and $\nabla_{\mathbf x_t}\log q_t(\mathbf x_t)$ with coefficients being each minus standard deviation respectively, and empirically we use the same fixed variance for both $q(\mathbf x_{t-1}|\mathbf x_t,\mathbf x_0,c)$ and $q(\mathbf x_{t-1}|\mathbf x_t,\mathbf x_0)$.
Therefore, \Cref{thm:cfg} is consistent with the original CFG using score functions in \cref{eq:cfg}.
\end{remark}

\subsection{Misconceptions on Expectation Shift}\label{subsec:method.2}

CFG is designed to concentrate on better exemplars for each denoising step by sharpening the gamma-powered distribution as below~\citep{bradley2024classifierfreeguidancepredictorcorrector}:
\begin{align}\label{eq:gamma_distrib}
q_{t,\gamma}(\mathbf x|c)=q_t(\mathbf x|c)^\gamma q_t(\mathbf x)^{1-\gamma}.
\end{align}
We first generalize the counterexample in~\citep{bradley2024classifierfreeguidancepredictorcorrector} to confirm the expectation shift phenomenon.
For VE-SDE with deterministic sampling recipe, we consider the 1-dimensional distribution with $q_0(\mathbf x_0|c)\sim\mathcal N(c,1)$, $q(c)\sim\mathcal N(0,1)$, $q_0(\mathbf x_0)\sim\mathcal N(0,2)$.
Then we can formulate the forward process and score functions as below:
\begin{align}
\hspace{-4mm}q_t(\mathbf x_t|c)\sim\mathcal N(c,1+t),&\;\nabla_{\mathbf x_t}\log q_t(\mathbf x_t|c)=-\frac{\mathbf x_t-c}{1+t}, \\
\hspace{-4mm}q_t(\mathbf x_t)\sim\mathcal N(0,2+t),&\;\nabla_{\mathbf x_t}\log q_t(\mathbf x_t)=-\frac{\mathbf x_t}{2+t}.
\end{align}

We state the theorem below describing the expectation shift.
Proof is addressed in \Cref{subsec:proof.1}.

\begin{theorem}\label{thm:gs_shift}
Denote by $q_{0,\gamma}^{\mathrm{deter}}(\mathbf x_0|c)$ the conditional distribution by solving PF-ODE in \cref{eq:pf_ode} with $\gamma>1$.
Then $q_{0,\gamma}^{\mathrm{deter}}(\mathbf x_0|c)$ follows the closed-form expression as below.
\begin{align}
q_{0,\gamma}^{\mathrm{deter}}(\mathbf x_0|c)\sim\mathcal N\left(c\phi(\gamma,T),2^{1-\gamma}\psi(\gamma,T)\right),
\end{align}
in which
\begin{align}
\phi(\gamma,T)&=\frac{2^{\frac{1-\gamma}{2}}}{(T+1)^{\frac{\gamma}{2}}(T+2)^{\frac{1-\gamma}{2}}} \\
&\qquad+\frac{\gamma}{2^{\frac{\gamma+1}{2}}}\int_0^T\frac{(s+1)^{-\frac{\gamma+2}{2}}}{(s+2)^{\frac{1-\gamma}{2}}}\mathrm ds, \\
\psi(\gamma,T)&=\frac{T+1}{(T+1)^\gamma(T+2)^{1-\gamma}}.
\end{align}

\input{figs_final/coeffs/fig.tex}

Specifically, when $T\rightarrow+\infty$, denote by $\phi(\gamma)$ with
\begin{align}
\phi(\gamma)=\lim_{T\rightarrow+\infty}\phi(\gamma,T),
\end{align}
we have $\phi(\gamma)\geqslant\gamma\frac{7}{15}\left(\frac{10}{7}\right)^{\frac{5-\gamma}{2}}$ for $\gamma\in[1,3]$, $\phi(1)=1$, $\phi(3)=2$, $\phi(\gamma)\geqslant2$ for all $\gamma>3$, and
\begin{align}
q_{0,\gamma}^{\mathrm{deter}}(\mathbf x_0|c)\sim\mathcal N(c\phi(\gamma),2^{1-\gamma}).
\end{align}
\end{theorem}

However, note that the ground-truth conditional distribution $q_0(\mathbf x_0|c)\sim\mathcal N(c,1)$, indicating that the ground-truth expectation is equal to $c$.
That is to say, denoising with CFG achieves at least twice as large expectation as the ground-truth one.
\cref{fig:toy_model} clearly describes the phenomenon.

\subsection{Rectified Classifier-Free Guidance}\label{subsec:method.3}

Recall that the constraint of the two coefficients with summation one disables the compatibility with diffusion theory and indicates expectation shift.
\Cref{thm:gs_shift} quantitatively describes the expectation shift, claiming that the two coefficients of conditional and unconditional score functions in \cref{eq:cfg} dominate both the expectation and variance of $q_{0,\gamma}^{\mathrm{deter}}(\mathbf x_0|c)$.
To this end, we propose to rectify CFG with relaxation on the guidance coefficients, \textit{i.e.},
\begin{align}\label{eq:generalized_cfg}
s_{t,\gamma_1,\gamma_0}(\mathbf x,c)&=\gamma_1\otimes\nabla_{\mathbf x}\log q_t(\mathbf x|c) \\
&\qquad+\gamma_0\otimes\nabla_{\mathbf x}\log q_t(\mathbf x),
\end{align}
in which $\gamma_1$, $\gamma_0\in\mathbb R^D$ are functions with respect to condition $c$ and timestep $t$, and $\otimes$ indicates element-wise product.
Denote by $q_{0,\gamma_1,\gamma_0}^{\mathrm{deter}}(\mathbf x_0|c)$ the attached conditional distribution following PF-ODE in \cref{eq:pf_ode} with $s_{t,\gamma_1,\gamma_0}(\mathbf x,c)$.

To make guided sampling compatible with the diffusion theory and annihilate expectation shift, it suffices to choose more appropriate $\gamma_1$ and $\gamma_0$ according to input condition $c$ and timestep $t$.
Intuitively, we need the constraint such that:
\begin{itemize}
\item Each component of $\gamma$ is larger than one for strengthened conditional fidelity, \textit{i.e.}, $\gamma_{1,i}>1$,
\item Denoising with PF-ODE and \cref{eq:generalized_cfg} is theoretically the reciprocal of forward process, thus $q_{0,\gamma_1,\gamma_0}^{\mathrm{deter}}(\mathbf x_0|c)$ enjoys the same expectation as the ground-truth $q_0(\mathbf x_0|c)$,
\item $q_{0,\gamma_1,\gamma_0}^{\mathrm{deter}}(\mathbf x_0|c)$ enjoys smaller or the same variance as the ground-truth $q_0(\mathbf x_0|c)$ for sharper distribution and thus concentrated better exemplars.
\end{itemize}

In the sequel, we omit $\otimes$ for simplicity.
We first focus on the compatibility with the diffusion theory.
We have claimed in \Cref{thm:cfg} that CFG cannot satisfy the diffusion theory due to nonzero $\mathbb E_{q_t(\mathbf x_t|c)}[\nabla_{\mathbf x_t}\log q_t(\mathbf x_t)]$.
To this end, it suffices to annihilate the expectation shift as below:
\begin{align}\label{eq:annihilation}
\mathbb E_{q_t(\mathbf x_t|c)}[s_{t,\gamma_1,\gamma_0}(\mathbf x,c)]=\mathbf0.
\end{align}
To confirm the feasibility and precisely describe the expectation of $q_{0,\gamma_1,\gamma_0}^{\mathrm{deter}}(\mathbf x_0|c)$, resembling \cref{eq:cfg_ddim1,eq:cfg_ddim2} we write denoised observation and denoising process as below:
\vspace{-12pt}
\begin{align}
\hat{\boldsymbol\epsilon}_\theta(\mathbf x_t,c,t)&=\gamma_1\boldsymbol\epsilon_\theta(\mathbf x_t,c,t)+\gamma_0\boldsymbol\epsilon_\theta(\mathbf x_t,t), \\
\hat{\mathbf f}_{\theta}^t(\mathbf x_t,c)&=\frac{1}{\alpha_t}(\mathbf x_t-\sigma_t\hat{\boldsymbol\epsilon}_\theta(\mathbf x_t,c,t)),\label{eq:generalized_cfg_ddim1} \\
\hspace{-2mm} \hat p_{\theta}(\mathbf x_{t-1}|\mathbf x_t,c)&=
\begin{cases}
\mathlarger{q_\delta(\mathbf x_{t-1}|\mathbf x_t,\hat{\mathbf f}_{\theta}^t(\mathbf x_t,c),c)}, & t>1, \\[2pt]
\mathlarger{\mathcal N(\hat{\mathbf f}_{\theta}^t(\mathbf x_1,c),\sigma_1^2\mathbf I)}, & t=1.
\end{cases}\label{eq:generalized_cfg_ddim2}
\end{align}
We have the theorem below, proof is in \Cref{subsec:proof.3}.

\begin{theorem}\label{thm:generalized_cfg_expectation}
Let $\mathbf x_t\sim q_t(\mathbf x_t|c)$, $\tilde{\mathbf x}_t\sim\hat p_{\theta}(\tilde{\mathbf x}_t|c)$ induced from DDIM sampler in \cref{eq:generalized_cfg_ddim2}.
Assume that all $\delta_t=0$, denote by $\Delta_t$ the difference between expectation of $\mathbf x_t$ and $\tilde{\mathbf x}_t$, by $\boldsymbol\epsilon_{\gamma_1,\gamma_0}^{c,t}$ the interpolation between score functions, \textit{i.e.},
\vspace{-12pt}
\begin{align}
\Delta_t&=\mathbb E_{q_t(\mathbf x_t|c)}[\mathbf x_t]-\mathbb E_{\hat p_{\theta}(\tilde{\mathbf x}_t|c)}[\tilde{\mathbf x}_t], \\
\boldsymbol\epsilon_{\gamma_1,\gamma_0}^{c,t}(\mathbf x)&=(\gamma_1-1)\boldsymbol\epsilon_\theta(\mathbf x,c,t)+\gamma_0\boldsymbol\epsilon_\theta(\mathbf x,t).
\end{align}
Then we have the following recursive equality:
\begin{align}
\hspace{-3mm} \Delta_{t-1}=\frac{\sigma_{t-1}}{\sigma_t}\Delta_t-(\sigma_{t-1}-\frac{\alpha_{t-1}}{\alpha_t}\sigma_t)\mathbb E_{\tilde{\mathbf x}_t}[\boldsymbol\epsilon_{\gamma_1,\gamma_0}^{c,t}(\tilde{\mathbf x}_t)].
\end{align}
Specifically, when $\Delta_t=0$, we have:
\begin{align}\label{eq:expectation_diff}
\Delta_{t-1}&=-(\sigma_{t-1}-\frac{\alpha_{t-1}}{\alpha_t}\sigma_t)\mathbb E_{\mathbf x_t}[\boldsymbol\epsilon_{\gamma_1,\gamma_0}^{c,t}(\mathbf x_t)].
\end{align}
\end{theorem}

\Cref{thm:generalized_cfg_expectation} studies the difference between expectation of denoising with \cref{eq:generalized_cfg} and the ground-truth.
Note that
\begin{align}
\mathbb E_{\mathbf x_t}[\boldsymbol\epsilon_\theta(\mathbf x_t,c,t)]=\mathbb E_{\mathbf x_t}[\mathbb E[\boldsymbol\epsilon|\mathbf x_t]]=\mathbb E_{\mathbf x_t}[\boldsymbol\epsilon]=\mathbf0,
\end{align}
therefore we have
\begin{align}
&\;\mathbb E_{\mathbf x_t}[\boldsymbol\epsilon_{\gamma_1,\gamma_0}^{c,t}(\mathbf x_t)] \\
=&\;\mathbb E_{\mathbf x_t}[(\gamma_1-1)\boldsymbol\epsilon_\theta(\mathbf x_t,c,t)+\gamma_0\boldsymbol\epsilon_\theta(\mathbf x_t,t)] \\
=&\;\mathbb E_{\mathbf x_t}[\gamma_1\boldsymbol\epsilon_\theta(\mathbf x_t,c,t)+\gamma_0\boldsymbol\epsilon_\theta(\mathbf x_t,t)],
\end{align}
which coincides with \cref{eq:annihilation}, indicating the feasibility and a closed-form solution given $c$ and $t$ as below:
\begin{align}\label{eq:closed_form}
\gamma_0=(1-\gamma_1)\mathbb E_{\mathbf x_t}[\boldsymbol\epsilon_\theta(\mathbf x_t,c,t)]/\mathbb E_{\mathbf x_t}[\boldsymbol\epsilon_\theta(\mathbf x_t,t)].
\end{align}

As for variance, however, normally we cannot analytically calculate the variance of $\hat p_{\theta}(\mathbf x_t|c)$.
Instead, we study the variance of toy data in \cref{subsec:method.2} as an empirical evidence in the following theorem, where proof is in \Cref{subsec:proof.4}.

\begin{theorem}\label{thm:generalized_cfg_variance}
Under settings in \Cref{thm:gs_shift}, denote by $q_{0,\gamma_1,\gamma_0}^{\mathrm{deter}}(\mathbf x_0|c)$ the conditional distribution by PF-ODE with $\gamma_1$ and $\gamma_0$ as in \cref{eq:generalized_cfg}. Then we have
\begin{align}
\mathrm{var}_{q_{0,\gamma_1,\gamma_0}^{\mathrm{deter}}(\mathbf x_0|c)}[\mathbf x_0]=2^{\gamma_0}(T+1)^{1-\gamma_1}(T+2)^{-\gamma_0}.
\end{align}
\end{theorem}

\input{figs_final/visualization/fig.tex}

According to \Cref{thm:generalized_cfg_variance}, it is noteworthy that variance of $q_{0,\gamma_1,\gamma_0}^{\mathrm{deter}}(\mathbf x_0|c)$ under toy setting is guaranteed to be smaller than the ground-truth $\mathrm{var}_{q_0(\mathbf x_0|c)}[\mathbf x_0]=1$ when each component satisfies that $\gamma_{0,i}\leqslant0$ and $\gamma_{1,i}+\gamma_{0,i}\geqslant1$, especially when $T\rightarrow+\infty$.

Now we formally propose the constraints.
First, we need each component $\gamma_{1,i}>1$ for strengthened conditional fidelity.
Then for expectation, it is noteworthy that $\Delta_T=0$ satisfies the assumption in \Cref{thm:generalized_cfg_expectation}.
Therefore by induction, it is feasible to annihilate $\Delta_0$ by annihilation of \cref{eq:annihilation} at all intermediate timesteps $t$.
Finally as for variance, we empirically set $\gamma_{0,i}\leqslant0$ and $\gamma_{1,i}+\gamma_{0,i}\geqslant0$.

Practically, we can determine $\gamma_0$ according to the guidance strength $\gamma_1$, condition $c$, and timestep $t$, according to the closed-form solution in \cref{eq:closed_form}.
Concretely, given condition $c$, it is feasible to pre-compute a collection of $\{(\epsilon_\theta(\mathbf x_t,c,t),\boldsymbol\epsilon_\theta(\mathbf x_t,t))\}_t$ by traversing $q_0(\mathbf x_0|c)$, and maintain a lookup table consisting of $\mathbb E_{\mathbf x_t}[\boldsymbol\epsilon_\theta(\mathbf x_t,c,t)]/\mathbb E_{\mathbf x_t}[\boldsymbol\epsilon_\theta(\mathbf x_t,t)]$.
Then given any $\gamma_1$, we can directly achieve $\gamma_0$ by multiplying $-(\gamma_1-1)$ with the expectation ratio.
Pseudo-code is addressed in \Cref{sec:supp_exp}.

We make further discussion about \method.
By Cauchy-Schwarz inequality and \cref{eq:expectation_diff} we have:
\begin{align}
\|\Delta_{t-1}\|_2^2&\leqslant(\sigma_{t-1}-\frac{\alpha_{t-1}}{\alpha_t}\sigma_t)^2\mathbb E_{\mathbf x_t}[\|\boldsymbol\epsilon_{\gamma_1,\gamma_0}^{c,t}(\mathbf x_t)\|_2^2].
\end{align}
Then we can define the objective resembling DPMs as below, optimizing reversely from $t=T$ to $0$.
\begin{align}\label{eq:objective}
\mathcal L=\mathbb E_{\mathbf x_t,t}[\|(\gamma_1-1)\boldsymbol\epsilon_\theta(\mathbf x_t,c,t)+\gamma_0\boldsymbol\epsilon_\theta(\mathbf x_t,t)\|_2^2].
\end{align}
Resembling \Cref{thm:cfg}, with \cref{eq:objective}, we can also show the compatibility of \method with DDIM, which is summarized as the theorem below.
Proof is addressed in \Cref{subsec:proof.5}

\begin{theorem}\label{thm:generalized_cfg}
For any $\{\delta_t\}_t$, \method with $\mathcal L$ is compatible with native DPM up to a constant.
\end{theorem}

%% file: figs_final/coeffs/fig.tex
\begin{figure*}[t]
\centering
\includegraphics[width=0.9\textwidth]{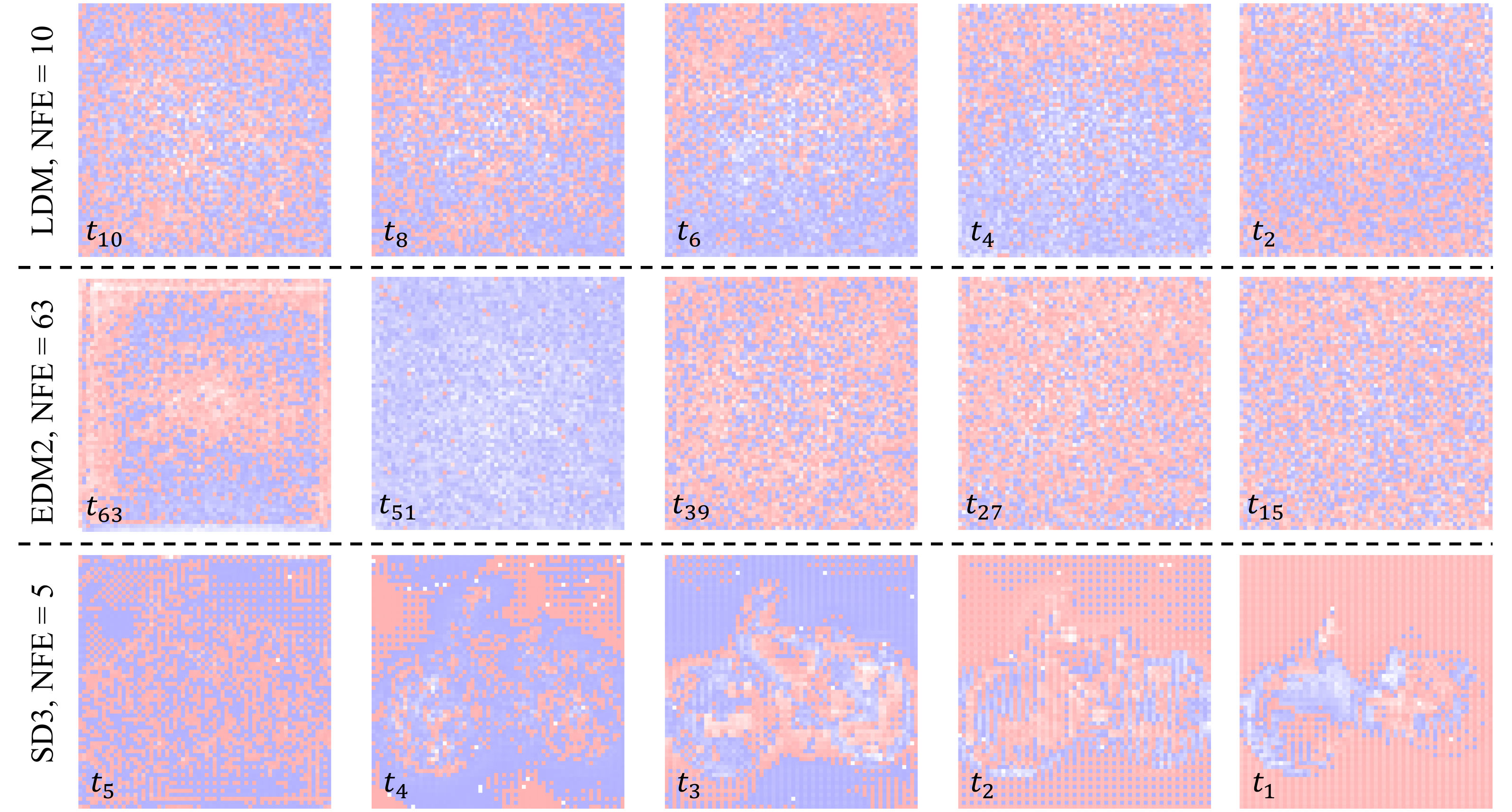}
\vspace{-5pt}
\caption{
    \textbf{Visualization} of the lookup table on LDM~\citep{rombach2022high}, EDM2~\citep{Karras2024edm2}, and SD3~\citep{esser2024sd3}, each of which consists of the expectation ratio $\mathbb E_{\mathbf x_t}[\boldsymbol\epsilon_\theta(\mathbf x_t,c,t)]/\mathbb E_{\mathbf x_t}[\boldsymbol\epsilon_\theta(\mathbf x_t,t)]$.
    Each pixel represents the scale of the pixel-wise ratio, \textit{i.e.}, color \textcolor{red}{\textbf{red}} implies that ratio is greater than one, while color \textcolor{blue}{\textbf{blue}} stands for ratio smaller than one.
    The darker the color is, the farther the ratio appears away from one.
    We report in each row the expectation ratios on five timesteps uniformly sampled from the whole trajectory, under different DPMs and NFEs.
    It is noteworthy that expectation ratios at the same timestep vary largely by different pixels, and there is no general pattern along with timesteps or pixels.
}
\vspace{-10pt}
\label{fig:coeff}
\end{figure*}

%% file: figs_final/visualization/fig.tex
\begin{figure*}[!tp]
\centering
\includegraphics[width=0.9\textwidth]{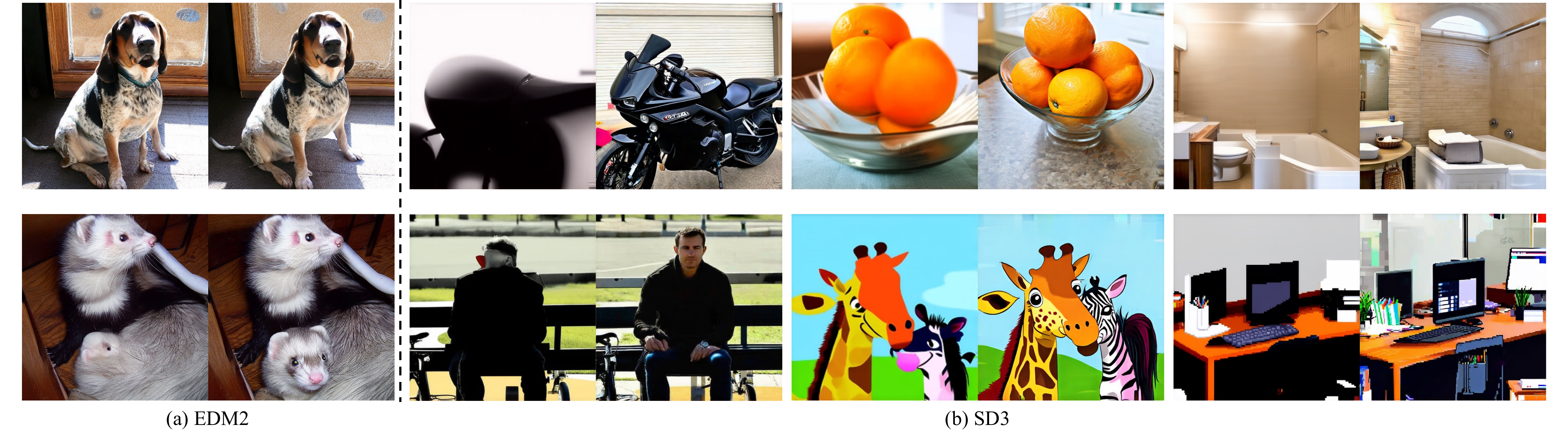}
\vspace{-5pt}
\caption{
    \textbf{Qualitative comparison} on EDM2 and SD3.
    Left and right in each cell suggest samples via CFG and \method, respectively.
}
\vspace{-5pt}
\label{fig:visualization}
\end{figure*}

%% file: sections_final/4_exp.tex
\section{Experiments}\label{sec:exp}

\input{tables_final/comparison_quan.tex}
\input{tables_final/variance.tex}

\subsection{Experimental Setups}\label{subsec:exp.1}

\noindent\textbf{Datasets and baselines.}
We apply \method to seminal class-conditioned and text-conditioned DPMs, including LDM~\citep{rombach2022high} and DiT~\citep{peebles2023dit} on ImageNet 256~\citep{dengjia2009}, EDM2~\citep{Karras2024edm2} on ImageNet 512, and SD3~\citep{esser2024sd3} on CC12M~\citep{changpinyo2021cc12m}, respectively.

\noindent\textbf{Evaluation metrics.}
As for class-conditioned LDM, DiT, and EDM2, we draw 50,000 samples for Fr\'{e}chet Inception Distance (FID)~\citep{heusel2017gans} and FD$_{\text{DINOv2}}$~\citep{stein2023fddino} to evaluate the fidelity and global coherency of the synthesized images, respectively.
We further use Improved Precision (Prec.) and Recall (Rec.)~\citep{Kynknniemi2019ImprovedPA} to separately measure sample fidelity (Precision) and diversity (Recall).
As for text-conditioned SD3, following the official implementation, we use CLIP Score (CLIP-S)~\citep{radford2021clip,hessel2021clipscore}, FID, and FD$_{\text{DINOv2}}$ on CLIP features~\citep{Sauer2021NEURIPS} on 1,000 samples to evaluate conditional faithfulness and fidelity of the synthesized images, respectively.
We also use MPS~\citep{zhang2024mps} to evaluate aesthetic scores.
All four metrics are evaluated on the same MS-COCO validation split~\citep{lin2015coco} as in official implementation~\citep{esser2024sd3}.

\noindent\textbf{Implementation details.}
We implement \method with NVIDIA A100 GPUs, and employ pre-trained LDM\footnote{https://github.com/CompVis/latent-diffusion}, DiT\footnote{https://github.com/facebookresearch/DiT}, EDM2\footnote{https://github.com/NVlabs/edm2}, and SD3\footnote{https://huggingface.co/stabilityai/stable-diffusion-3-medium-diffusers} checkpoints provided in official implementation.
We reproduce all the experiments with official and more other configurations including NFEs and guidance strengths.

\subsection{Results on Toy Example in \Cref{subsec:method.2}}\label{subsec:exp.2}

We first confirm the effectiveness of our method on toy data, as presented in \cref{subsec:method.2}.
Given the closed-form expressions of score functions, we are able to precisely describe the distributions of both gamma-powered distribution $q_{0,\gamma}(\mathbf x_0|c)$ by native CFG and $q_{0,\gamma_1,\gamma_0}^{\mathrm{deter}}(\mathbf x_0|c)$ by our \method.
The theoretical and numerical DDIM-based simulation value of probability density functions of both $q_{0,\gamma}(\mathbf x_0|c)$ and $q_{0,\gamma_1,\gamma_0}^{\mathrm{deter}}(\mathbf x_0|c)$ are shown in \cref{fig:toy_model}.
It is noteworthy that native CFG drifts the expectation of $q_{0,\gamma}(\mathbf x_0|c)$ further away from the peak of the ground-truth $q_0(\mathbf x_0|c)$ as $\gamma$ becomes larger, consistent with \Cref{thm:gs_shift}.
As a comparison, the peaks of $q_{0,\gamma_1,\gamma_0}^{\mathrm{deter}}(\mathbf x_0|c)$ and $q_0(\mathbf x_0|c)$ coincide, while $q_{0,\gamma_1,\gamma_0}^{\mathrm{deter}}(\mathbf x_0|c)$ is sharpened with smaller variance.
Therefore, by adopting relaxation on coefficients $\gamma_1$ and $\gamma_0$ with specially proposed constraints, our \method manages to annihilate expectation shift, enabling better guidance and thus better conditional fidelity.

\subsection{Results on Real Datasets}\label{subsec:exp.3}

We showcase some results in \cref{fig:visualization}.
One can see that \method could fix artifacts on EDM2.
It is also noteworthy that \method significantly improves synthesis quality on SD3, especially detailed textures.
Beyond the exhibited visualization, we conduct extensive quantitative experiments on state-of-the-art DPMs to further convey the efficacy of \method.
From \cref{tab:imagenet,tab:cc12m}, we can tell that \method is capable of better performance on both class-conditioned and text-conditioned DPMs under various guidance strengths and NFEs especially CLIP-S, indicating better conditional fidelity on open-vocabulary synthesis.
Furthermore, correction for theoretical flaws of CFG enables strong compatibility of \method with other empirical strategies such as RescaleCFG~\citep{lin2024rescale}, achieving better performance.

\subsection{Analyses}\label{subsec:exp.4}

\input{tables_final/ablation.tex}

\noindent\textbf{Variance of lookup table.}
Note that we need to pre-compute the lookup table consisting of expectation ratios for all conditions $c$, which is time-consuming and impractical for open-vocabulary distributions (\textit{e.g.}, text-conditioned DPMs).
In \cref{tab:mean_std} we report the mean and variance of expectation ratios over condition $c$, which is averaged on all timesteps and pixels.
One can observe that larger NFE suggests smaller ratio with also smaller variance.
It is also noteworthy that the variance of text-conditioned DPMs is larger than that of class-conditioned ones due to far more complex open-vocabulary conditions, while both of which is insignificant compared to the mean.
Therefore, it is feasible to prepare the lookup table for only part of all potential conditions and use the mean for all conditions, serving as a practically adequate strategy to improve time efficiency.
Variance over timestep $t$ averaged on all pixels and part of conditions is also reported in \cref{tab:mean_std}, where similar conclusion could be achieved.

\noindent\textbf{Ablation studies.}
Recall that we pre-compute the lookup table by traversing $q_0(\mathbf x_0|c)$ for each condition $c$.
Comprehensive ablation studies reported in \cref{tab:ablation_imagenet,tab:ablation_cc12m} convey a direct and clear picture of the efficacy of \method under different numbers of traversals.
We can conclude that larger number of traversals suggests better guidance performance, yet improvements from 100 to 500 traversals are relatively inconspicuous, especially on text-conditoned DPMs.
In other words, employing 500 samples per condition is adequate in practice to serve as an empirical setting.

\noindent\textbf{Time cost.}
Given the analyses on variance of expectation ratios over condition $c$ and ablation on traversals, preparing the lookup table is quite efficient.
In practice, we sample 500 images for a subset of 100 conditions, which takes $\sim3$ hours using 1 NVIDIA A100 GPU.
The time cost is very close to performing FID evaluation using 50,000 images.

\noindent\textbf{Pixel-wise lookup table.}
\method enables pixel-specific guidance coefficients $\gamma_1$ and $\gamma_0$ with the same shape as score functions, thanks to the closed-form solution in \cref{eq:closed_form}, \textit{i.e.}, we can assign $\gamma_0$ for each pixel by maintaining the lookup table of pixel-wise expectation ratios.
\cref{fig:coeff} demonstrates the ratios on LDM, EDM2, and SD3 at uniformly sampled timesteps under different NFEs.
Both LDM and EDM2 show the generation of class 0 in ImageNet (\textit{i.e.}, ``tench''), while SD3 adopts the prompt ``A bicycle replica with a clock as the front wheel''.
One can observe that expectation ratios at the same timestep vary largely by different pixels, and there appears no general rules on the relation between $\gamma_1$ and $\gamma_0$.
Therefore, it is indicated that trivially setting $\gamma_1$ and $\gamma_0$ to be scalars is less reasonable.
As a comparison, our method makes it possible to employ more precise control on guided sampling in a simple and post-hoc fashion without further fine-tuning, enabling better performance.
It is also noteworthy that expectation ratio of SD3 exhibits noticeable shapes, probably due to more informative text prompts than one-hot class labels and more powerful model thanks to the training scale.

\subsection{Discussions}\label{subsec:exp.5}

Classifier-Free Guidance is designed from Bayesian theory to facilitate conditional sampling, yet appears incompatible with original diffusion theory.
Therefore, we believe \method is attached to great importance on guided sampling by fixing the theoretical flaw of CFG.
Despite the success on better conditional fidelity, our algorithm has several potential limitations.
Theoretically, we need to pre-compute the lookup table by traversing the dataset to achieve rectified coefficients for each condition.
Although we conduct extensive ablation studies on the number of traversals and variance over condition $c$, providing an adequate strategy especially for open-vocabulary datasets on text-conditioned synthesis, the optimal strategy is unexplored.
Besides, we at present cannot provide precise control on variance of \method due to incomputable variance in denoising process, and turn to employ empirical values.
Therefore, how to further conquer these problems (\textit{e.g.}, employing a predictor network $\boldsymbol\omega(c,t)$ for better $\gamma_0$ on open-vocabulary datasets according to \cref{eq:objective}) will be an interesting avenue for future research.
Although leaving the variance behavior unexplored, we hope that \method will encourage the community to close the gap in the future.

%% file: tables_final/comparison_quan.tex
\definecolor{mygreen}{RGB}{34,170,133}
\begin{table*}[t]
\begin{minipage}[t]{0.48\textwidth}
\caption{
    \textbf{Sample quality} on ImageNet~\citep{dengjia2009}.
}
\label{tab:imagenet}
\vspace{-7pt}
\centering
\SetTblrInner{rowsep=1.048pt}                % Row space.
\SetTblrInner{colsep=2.2pt}                  % Col space.
\scriptsize
\begin{tblr}{
    cell{2-22}{2-7}={halign=c,valign=m},         % Text alignment for all cells.
    cell{2-22}{1}={halign=l,valign=m},           % Text alignment for all cells.
    cell{24-30}{2-7}={halign=c,valign=m},        % Text alignment for all cells.
    cell{24-30}{1}={halign=l,valign=m},          % Text alignment for all cells.
    cell{1,23}{1-7}={halign=l,valign=m},         % Text alignment for all cells.
    cell{1,23}{1}={c=7}{},                       % Multi-row cells.
    cell{25,27,29}{1}={r=2}{},                   % Multi-column cells.
    hline{1,2,23,24,31}={1-7}{1.0pt},            % Horizontal lines.
    hline{3,5,7,9,11,13,15,17,19,21}={1-7}{},    % Horizontal lines.
    hline{25,27,29}={1-7}{},                     % Horizontal lines.
    % cell{3,5,7,9}{1-6}={bg=lightgray!35},    % Color gray.
    % cell{13,15}{1-6}={bg=lightgray!35},      % Color gray.
    % cell{17,19}{1-6}={bg=lightgray!35},      % Color gray.
}
\textbf{ImageNet 256x256} &                                  &                    &                                     &                    &                    &                   \\
                    Model &                           Method & NFE $(\downarrow)$ & FD$_{\text{DINOv2}}$ $(\downarrow)$ & FID $(\downarrow)$ & Prec. $(\uparrow)$ & Rec. $(\uparrow)$ \\
                 DiT-XL/2 &                              CFG &                250 &                              120.07 &               2.27 &               0.83 &              0.57 \\
          $\gamma_1=1.50$ & \textcolor{mygreen}{\bf \method} &                250 &                          \bf 118.71 &           \bf 2.13 &               0.83 &          \bf 0.58 \\
                 DiT-XL/2 &                              CFG &                250 &                              162.68 &               3.22 &               0.76 &              0.62 \\
          $\gamma_1=1.25$ & \textcolor{mygreen}{\bf \method} &                250 &                          \bf 145.79 &           \bf 3.01 &           \bf 0.77 &          \bf 0.63 \\
                      LDM &                              CFG &                 20 &                              180.60 &              18.87 &               0.95 &              0.15 \\
           $\gamma_1=5.0$ & \textcolor{mygreen}{\bf \method} &                 20 &                          \bf 169.41 &          \bf 16.95 &           \bf 0.91 &          \bf 0.18 \\
                      LDM &                              CFG &                 20 &                              149.79 &              11.46 &           \bf 0.94 &              0.27 \\
           $\gamma_1=3.0$ & \textcolor{mygreen}{\bf \method} &                 20 &                          \bf 142.54 &           \bf 9.78 &               0.91 &          \bf 0.32 \\
                      LDM &                              CFG &                 20 &                              152.51 &               5.32 &               0.88 &              0.42 \\
           $\gamma_1=2.0$ & \textcolor{mygreen}{\bf \method} &                 20 &                          \bf 149.91 &           \bf 4.40 &               0.88 &          \bf 0.45 \\
                      LDM &                              CFG &                 20 &                              203.17 &               5.36 &               0.80 &              0.51 \\
           $\gamma_1=1.5$ & \textcolor{mygreen}{\bf \method} &                 20 &                          \bf 198.44 &           \bf 4.78 &               0.80 &          \bf 0.53 \\
                      LDM &                              CFG &                 10 &                              156.41 &              16.78 &           \bf 0.94 &              0.16 \\
           $\gamma_1=5.0$ & \textcolor{mygreen}{\bf \method} &                 10 &                          \bf 150.47 &          \bf 14.46 &               0.89 &          \bf 0.22 \\
                      LDM &                              CFG &                 10 &                              153.97 &              10.13 &               0.91 &              0.28 \\
           $\gamma_1=3.0$ & \textcolor{mygreen}{\bf \method} &                 10 &                          \bf 142.04 &           \bf 8.26 &               0.91 &          \bf 0.33 \\
                      LDM &                              CFG &                 10 &                              183.39 &               7.83 &               0.81 &              0.38 \\
           $\gamma_1=2.0$ & \textcolor{mygreen}{\bf \method} &                 10 &                          \bf 182.04 &           \bf 5.98 &           \bf 0.83 &          \bf 0.42 \\
                      LDM &                              CFG &                 10 &                              251.07 &              13.19 &               0.69 &              0.46 \\
           $\gamma_1=1.5$ & \textcolor{mygreen}{\bf \method} &                 10 &                          \bf 248.23 &          \bf 11.27 &           \bf 0.72 &          \bf 0.49 \\
\textbf{ImageNet 512x512} &                                  &                    &                                     &                    &                    &                   \\
                    Model &                           Method & NFE $(\downarrow)$ & FD$_{\text{DINOv2}}$ $(\downarrow)$ & FID $(\downarrow)$ & Prec. $(\uparrow)$ & Rec. $(\uparrow)$ \\
                   EDM2-S &                              CFG &                 63 &                               52.32 &               2.29 &               0.83 &              0.59 \\
                          & \textcolor{mygreen}{\bf \method} &                 63 &                           \bf 50.56 &           \bf 2.23 &               0.83 &              0.59 \\
                   EDM2-M &                              CFG &                 63 &                               41.98 &               2.12 &               0.81 &              0.60 \\
                          & \textcolor{mygreen}{\bf \method} &                 63 &                           \bf 41.55 &           \bf 2.06 &               0.81 &              0.61 \\
                   EDM2-L &                              CFG &                 63 &                               38.20 &               1.96 &               0.81 &              0.62 \\
                          & \textcolor{mygreen}{\bf \method} &                 63 &                           \bf 36.75 &           \bf 1.89 &               0.81 &              0.62 \\
\end{tblr}
\end{minipage}
\hfill
\begin{minipage}[t]{0.48\textwidth}
\caption{
    \textbf{Sample quality} on CC12M~\citep{changpinyo2021cc12m}
}
\label{tab:cc12m}
\vspace{-7pt}
\centering
\SetTblrInner{rowsep=1.916pt}                % Row space.
\SetTblrInner{colsep=2.5pt}                  % Col space.
\scriptsize
\begin{tblr}{
    cell{2-26}{2-6}={halign=c,valign=m},     % Text alignment for all cells.
    cell{2-26}{1}={halign=l,valign=m},       % Text alignment for all cells.
    cell{1}{1-6}={halign=l,valign=m},        % Text alignment for all cells.
    cell{1}{1}={c=6}{},                      % Multi-row cells.
    hline{1,2,27}={1-6}{1.0pt},              % Horizontal lines.
    hline{3,7,11,15,19,23,27}={1-6}{},       % Horizontal lines.
    % cell{3,5,7}{1-5}={bg=lightgray!35},      % Color gray.
    % cell{9,11,13}{1-5}={bg=lightgray!35},    % Color gray.
    % cell{15,17,19}{1-5}={bg=lightgray!35},   % Color gray.
}
\textbf{CC12M 512x512, SD3~\citep{esser2024sd3}} &            &                    &                     &                                     &                  \\
                                          Method & $\gamma_1$ & NFE $(\downarrow)$ & CLIP-S $(\uparrow)$ & FD$_{\text{DINOv2}}$ $(\downarrow)$ & MPS $(\uparrow)$ \\
                                             CFG &        7.5 &                 10 &               0.262 &                             1105.51 &            9.828 \\
                \textcolor{mygreen}{\bf \method} &        7.5 &                 10 &               0.263 &                             1010.14 &           10.250 \\
               RescaleCFG~\citep{lin2024rescale} &        7.5 &                 10 &               0.267 &                             1011.62 &           11.258 \\
   RescaleCFG \textcolor{mygreen}{\bf + \method} &        7.5 &                 10 &           \bf 0.268 &                          \bf 979.87 &       \bf 11.336 \\
                                             CFG &        5.0 &                 10 &               0.268 &                             1053.44 &           10.883 \\
                \textcolor{mygreen}{\bf \method} &        5.0 &                 10 &           \bf 0.269 &                              999.48 &           11.031 \\
               RescaleCFG~\citep{lin2024rescale} &        5.0 &                 10 &               0.267 &                             1009.67 &           11.242 \\
   RescaleCFG \textcolor{mygreen}{\bf + \method} &        5.0 &                 10 &           \bf 0.269 &                          \bf 984.25 &       \bf 11.297 \\
                                             CFG &        2.5 &                 10 &               0.265 &                             1016.79 &           10.367 \\
                \textcolor{mygreen}{\bf \method} &        2.5 &                 10 &               0.265 &                              977.39 &           10.438 \\
               RescaleCFG~\citep{lin2024rescale} &        2.5 &                 10 &               0.265 &                             1003.64 &           10.445 \\
   RescaleCFG \textcolor{mygreen}{\bf + \method} &        2.5 &                 10 &           \bf 0.266 &                          \bf 963.21 &       \bf 10.477 \\
                                             CFG &        7.5 &                  5 &               0.209 &                             1466.91 &            3.189 \\
                \textcolor{mygreen}{\bf \method} &        7.5 &                  5 &               0.229 &                             1323.49 &            3.979 \\
               RescaleCFG~\citep{lin2024rescale} &        7.5 &                  5 &           \bf 0.258 &                             1114.92 &            8.102 \\
   RescaleCFG \textcolor{mygreen}{\bf + \method} &        7.5 &                  5 &           \bf 0.258 &                         \bf 1070.65 &        \bf 8.219 \\
                                             CFG &        5.0 &                  5 &               0.248 &                             1218.18 &            6.484 \\
                \textcolor{mygreen}{\bf \method} &        5.0 &                  5 &               0.258 &                             1074.60 &            7.398 \\
               RescaleCFG~\citep{lin2024rescale} &        5.0 &                  5 &               0.265 &                             1087.98 &            8.719 \\
   RescaleCFG \textcolor{mygreen}{\bf + \method} &        5.0 &                  5 &           \bf 0.266 &                         \bf 1040.01 &        \bf 8.813 \\
                                             CFG &        2.5 &                  5 &               0.261 &                             1119.06 &            7.902 \\
                \textcolor{mygreen}{\bf \method} &        2.5 &                  5 &           \bf 0.263 &                             1058.86 &            8.172 \\
               RescaleCFG~\citep{lin2024rescale} &        2.5 &                  5 &               0.262 &                             1093.85 &            8.133 \\
   RescaleCFG \textcolor{mygreen}{\bf + \method} &        2.5 &                  5 &           \bf 0.263 &                         \bf 1041.53 &        \bf 8.266 \\
\end{tblr}
\end{minipage}
\vspace{-2pt}
\end{table*}

%% file: tables_final/variance.tex
\begin{table*}[t]
\caption{
    \textbf{Variance} of lookup table over condition $c$ and timestep $t$.
    Note that we employ pixel-wise lookup table involving both $c$ and $t$.
    We report the the mean and variance of lookup table over $c$ and $t$, respectively, which are computed by averaging on all pixels.
}
\label{tab:mean_std}
\vspace{-4pt}
\centering
\SetTblrInner{rowsep=1.248pt}                % Row space.
\SetTblrInner{colsep=18.0pt}                 % Col space.
\footnotesize
\begin{tblr}{
    cell{1-3}{2-5}={halign=c,valign=m},      % Text alignment for all cells.
    cell{1-3}{1}={halign=l,valign=m},        % Text alignment for all cells.
    hline{1,4}={1-5}{1.0pt},                 % Horizontal lines.
    hline{2}={1-5}{},                        % Horizontal lines.
}
Config.           &     LDM, NFE $=$ 10 &    EDM2, NFE $=$ 63 &      SD3, NFE $=$ 5 &     SD3, NFE $=$ 10 \\
Variance over $c$ & 1.0050 $\pm$ 0.0012 & 1.0060 $\pm$ 0.0119 & 1.0250 $\pm$ 0.0369 & 1.0125 $\pm$ 0.0281 \\
Variance over $t$ & 1.0050 $\pm$ 0.0013 & 1.0060 $\pm$ 0.1545 & 1.0250 $\pm$ 0.0359 & 1.0125 $\pm$ 0.0306 \\
\end{tblr}
\vspace{-10pt}
\end{table*}

%% file: tables_final/ablation.tex
\begin{table*}[t]
\begin{minipage}[t]{0.48\textwidth}
\caption{
    \textbf{Ablation study} of the number of traversals (the number after \method) for lookup table on ImageNet~\citep{dengjia2009}.
    For clearer demonstration, baselines of native CFG are highlighted in \textbf{\textcolor{gray}{gray}}.
}
\label{tab:ablation_imagenet}
\vspace{-4pt}
\centering
\SetTblrInner{rowsep=1.138pt}                % Row space.
\SetTblrInner{colsep=5.2pt}                  % Col space.
\footnotesize
\begin{tblr}{
    cell{2-6}{1-6}={halign=c,valign=m},      % Text alignment for all cells.
    cell{8-16}{1-6}={halign=c,valign=m},     % Text alignment for all cells.
    cell{1,7}{1-6}={halign=l,valign=m},      % Text alignment for all cells.
    cell{1,7}{1}={c=6}{},                    % Multi-row cells.
    hline{1,2,7,8,17}={1-6}{1.0pt},          % Horizontal lines.
    hline{3,9}={1-6}{},                      % Horizontal lines.
    cell{3,9}{1-6}={bg=lightgray!35},        % Color gray.
}
\textbf{ImageNet 256x256, LDM~\citep{rombach2022high}}   &             &                    &                    &                    &                   \\
$\gamma_1$                                               & $\gamma_0$  & NFE $(\downarrow)$ & FID $(\downarrow)$ & Prec. $(\uparrow)$ & Rec. $(\uparrow)$ \\
                                                     3.0 &        -2.0 &                 10 &              10.13 &               0.91 &              0.28 \\
                                                     3.0 &  \method-10 &                 10 &               8.88 &           \bf 0.92 &              0.30 \\
                                                     3.0 & \method-100 &                 10 &               8.70 &           \bf 0.92 &              0.31 \\
                                                     3.0 & \method-500 &                 10 &           \bf 8.26 &               0.91 &          \bf 0.33 \\
\textbf{ImageNet 512x512, EDM2-S~\citep{Karras2024edm2}} &             &                    &                    &                    &                   \\
$\gamma_1$                                               & $\gamma_0$  & NFE $(\downarrow)$ & FID $(\downarrow)$ & Prec. $(\uparrow)$ & Rec. $(\uparrow)$ \\
                                                     2.5 &        -1.5 &                 63 &               5.87 &           \bf 0.85 &              0.46 \\
                                                     2.5 &  \method-10 &                 63 &               5.06 &               0.84 &              0.47 \\
                                                     2.5 & \method-100 &                 63 &               4.99 &               0.84 &              0.45 \\
                                                     2.5 & \method-500 &                 63 &           \bf 4.84 &               0.84 &          \bf 0.48 \\
                                                     2.0 &        -1.0 &                 63 &               4.18 &           \bf 0.85 &              0.52 \\
                                                     2.0 &  \method-10 &                 63 &               3.70 &               0.84 &              0.52 \\
                                                     2.0 & \method-100 &                 63 &               3.66 &               0.84 &              0.52 \\
                                                     2.0 & \method-500 &                 63 &           \bf 3.61 &               0.84 &              0.52 \\
\end{tblr}
\end{minipage}
\hfill
\begin{minipage}[t]{0.48\textwidth}
\caption{
    \textbf{Ablation study} of the number of traversals (the number after \method) for lookup table on CC12M~\citep{changpinyo2021cc12m}.
    For clearer demonstration, baselines of native CFG are highlighted in \textbf{\textcolor{gray}{gray}}.
}
\label{tab:ablation_cc12m}
\vspace{-4pt}
\centering
\SetTblrInner{rowsep=2.089pt}                % Row space.
\SetTblrInner{colsep=8.1pt}                  % Col space.
\footnotesize
\begin{tblr}{
    cell{2-14}{1-5}={halign=c,valign=m},     % Text alignment for all cells.
    cell{1}{1-5}={halign=l,valign=m},        % Text alignment for all cells.
    cell{1}{1}={c=5}{},                      % Multi-row cells.
    hline{1,2,15}={1-5}{1.0pt},              % Horizontal lines.
    hline{3}={1-5}{},                        % Horizontal lines.
    cell{3,7,11}{1-5}={bg=lightgray!35},     % Color gray.
}
\textbf{CC12M 512x512, SD3~\citep{esser2024sd3}} &             &                    &                     &                    \\
$\gamma_1$                                       &  $\gamma_0$ & NFE $(\downarrow)$ & CLIP-S $(\uparrow)$ & FID $(\downarrow)$ \\
                                             5.0 &        -4.0 &                 25 &               0.267 &              72.37 \\
                                             5.0 &  \method-10 &                 25 &               0.267 &              72.15 \\
                                             5.0 & \method-100 &                 25 &           \bf 0.268 &              72.03 \\
                                             5.0 & \method-500 &                 25 &           \bf 0.268 &          \bf 71.95 \\
                                             5.0 &        -4.0 &                 10 &               0.268 &              72.55 \\
                                             5.0 &  \method-10 &                 10 &               0.268 &              71.61 \\
                                             5.0 & \method-100 &                 10 &               0.268 &              70.64 \\
                                             5.0 & \method-500 &                 10 &           \bf 0.269 &          \bf 70.31 \\
                                             5.0 &        -4.0 &                  5 &               0.248 &             115.51 \\
                                             5.0 &  \method-10 &                  5 &               0.252 &             107.09 \\
                                             5.0 & \method-100 &                  5 &               0.256 &             103.25 \\
                                             5.0 & \method-500 &                  5 &           \bf 0.258 &         \bf 101.82 \\
\end{tblr}
\end{minipage}
\vspace{-8pt}
\end{table*}

%% file: sections_final/5_conclusion.tex
\section{Conclusion}\label{sec:conclusion}

In this paper, we analyze the theoretical flaws of native Classifier-Free Guidance technique and the induced expectation shift phenomenon.
We theoretically claim the exact value of expectation shift on a toy distribution.
Introducing a relaxation on coefficients of CFG and novel constraints, we manage to complete the theory of guided sampling by fixing the incompatibility between CFG and diffusion theory.
Accordingly, thanks to the closed-form solution to the constraints, we propose \method, a post-hoc algorithm aiming at more faithful guided sampling by determining the coefficients from a pre-computed lookup table.
We further study the behavior of the lookup table, proposing an adequate strategy for better time efficiency in practice.
Comprehensive experiments demonstrate the efficacy of our method on various state-of-the-art DPMs under different NFEs and guidance strengths.

%% file: sections_final/8_acknowledgement.tex
\section*{Acknowledgements}

This work was partially supported by the Natural Science Foundation of China (62461160309, 62302297), Beijing Science and Technology plan project (Z231100005923029), NSFC-RGC Joint Research Scheme (N\_HKU705/24), Ant Group Research Intern Program, and Shanghai Sailing Program (22YF1420300).

%% file: sections_final/6_ref.tex
{
\small
\bibliographystyle{ieeenat_fullname}
\bibliography{ref.bib}
}

%% file: sections_final/7_appendix.tex
\clearpage
\appendix
\renewcommand\thesection{\Alph{section}}
\renewcommand\thefigure{S\arabic{figure}}
\renewcommand\thetable{S\arabic{table}}
\renewcommand\theequation{S\arabic{equation}}
\setcounter{figure}{0}
\setcounter{table}{0}
\setcounter{equation}{0}
\setcounter{page}{1}
\maketitlesupplementary
% \onecolumn

\section*{Appendix}

\section{Proofs and derivations}\label{sec:proof}

In this section, we will prove the theorems stated in the main manuscript.

\subsection{Proof of \Cref{thm:cfg}}\label{subsec:proof.2}

We first claim two lemmas which are crucial for the proof.

\begin{lemma}\label{lem:1}
Let $g(\mathbf x_t)$ and $h(\mathbf x_t,\boldsymbol\epsilon)$ be integrable functions, then the following equality holds.
\begin{align}
\mathbb E_{q(\mathbf x)}[\langle g(\mathbf x),\mathbb E_{q(\boldsymbol\epsilon|\mathbf x)}[h(\mathbf x,\boldsymbol\epsilon)|\mathbf x]\rangle]=\mathbb E_{q(\mathbf x,\boldsymbol\epsilon)}[\langle g(\mathbf x),h(\mathbf x,\boldsymbol\epsilon)\rangle],
\end{align}
in which $\langle\cdot,\cdot\rangle$ is inner product.
\end{lemma}

\begin{proof}[Proof of \Cref{lem:1}]
Note that
\begin{align}
\mathbb E_{q(\mathbf x)}[\langle g(\mathbf x),\mathbb E_{q(\boldsymbol\epsilon|\mathbf x)}[h(\mathbf x,\boldsymbol\epsilon)|\mathbf x]\rangle]&=\int\langle g(\mathbf x),\mathbb E_{q(\boldsymbol\epsilon|\mathbf x)}[h(\mathbf x,\boldsymbol\epsilon)|\mathbf x]\rangle q(\mathbf x)\mathrm d\mathbf x \\
&=\int\langle g(\mathbf x),\int h(\mathbf x,\boldsymbol\epsilon)q(\boldsymbol\epsilon|\mathbf x)\mathrm d\boldsymbol\epsilon\rangle q(\mathbf x)\mathrm d\mathbf x \\
&=\iint\langle g(\mathbf x),h(\mathbf x,\boldsymbol\epsilon)\rangle q(\mathbf x)q(\boldsymbol\epsilon|\mathbf x)\mathrm d\boldsymbol\epsilon\mathrm d\mathbf x\label{eq:linear} \\
&=\mathbb E_{q(\mathbf x,\boldsymbol\epsilon)}[\langle g(\mathbf x),h(\mathbf x,\boldsymbol\epsilon)\rangle],
\end{align}
in which \cref{eq:linear} is by linearity of integral.
\end{proof}

\begin{lemma}\label{lem:2}
The following equality of expectation holds:
\begin{align}
\mathbb E_{\mathbf x}[\boldsymbol\epsilon_\theta(\mathbf x,t)]=\frac{1}{\sigma_t}\mathbb E_{\mathbf x}[\mathbf x]-\frac{\alpha_t}{\sigma_t}\mathbb E_{c,\mathbf x_0,\mathbf x}[\mathbf x_0].
\end{align}
\end{lemma}

\begin{proof}[Proof of \Cref{lem:2}]
Note that
\begin{align}
\nabla_{\mathbf x}\log q_t(\mathbf x)&=\frac{\nabla_{\mathbf x}q_t(\mathbf x)}{q_t(\mathbf x)} \\
&=\frac{\nabla_{\mathbf x}\int q_t(\mathbf x|c)q(c)\mathrm dc}{q_t(\mathbf x)} \\
&=\frac{\int\nabla_{\mathbf x}q_t(\mathbf x|c)q(c)\mathrm dc}{q_t(\mathbf x)} \\
&=\frac{\int q_t(\mathbf x|c)q(c)\nabla_{\mathbf x}\log q_t(\mathbf x|c)\mathrm dc}{q_t(\mathbf x)} \\
&=\int\frac{q_t(\mathbf x|c)q(c)}{q_t(\mathbf x)}\nabla_{\mathbf x}\log q_t(\mathbf x|c)\mathrm dc \\
&=\mathbb E_{q_t(c|\mathbf x)}[\nabla_{\mathbf x}\log q_t(\mathbf x|c)|\mathbf x].
\end{align}
Therefore, we have
\begin{align}
\boldsymbol\epsilon_\theta(\mathbf x,t)&=\mathbb E_{q_t(c|\mathbf x)}[\boldsymbol\epsilon_\theta(\mathbf x,c,t)|\mathbf x] \\
&=\mathbb E_{q_t(c|\mathbf x)}\left[\mathbb E_{q(\mathbf x_0|\mathbf x,c)}\left[\frac{\mathbf x-\alpha_t\mathbf x_0}{\sigma_t}\right]|\mathbf x\right] \\
&=\mathbb E_{q_t(c,\mathbf x_0|\mathbf x)}\left[\frac{\mathbf x-\alpha_t\mathbf x_0}{\sigma_t}|\mathbf x\right] \\
&=\frac{1}{\sigma_t}\mathbf x-\frac{\alpha_t}{\sigma_t}\mathbb E_{q_t(c,\mathbf x_0|\mathbf x)}[\mathbf x_0|\mathbf x],
\end{align}
and
\begin{align}
\mathbb E_{\mathbf x}[\boldsymbol\epsilon_\theta(\mathbf x,t)]&=\frac{1}{\sigma_t}\mathbb E_{\mathbf x}[\mathbf x]-\frac{\alpha_t}{\sigma_t}\mathbb E_{\mathbf x}[\mathbb E_{q_t(c,\mathbf x_0|\mathbf x)}[\mathbf x_0|\mathbf x]] \\
&=\frac{1}{\sigma_t}\mathbb E_{\mathbf x}[\mathbf x]-\frac{\alpha_t}{\sigma_t}\mathbb E_{c,\mathbf x_0,\mathbf x}[\mathbf x_0].
\end{align}
\end{proof}

Then we start to prove \Cref{thm:cfg}.

\begin{proof}[Proof of \Cref{thm:cfg}]
Similar to derivation in DDIM~\citep{song2020denoising}, first rewrite $J_{\delta,\gamma}$ as below:
\begin{align}
J_{\delta,\gamma}=\mathbb E\left[-\log\hat p_{\theta}(\mathbf x_0|\mathbf x_1,c)+\sum_{t=2}^TD_{KL}(q_\delta(\mathbf x_{t-1}|\mathbf x_t,\mathbf x_0,c)\|\hat p_{\theta}(\mathbf x_{t-1}|\mathbf x_t,c))\right]+C_1,
\end{align}
in which $C_1$ is a constant not involving $\gamma$ and $\theta$.

Note that $\boldsymbol\epsilon_\theta(\mathbf x_t,c,t)=\mathbb E_{q(\boldsymbol\epsilon|\mathbf x_t,c)}[\boldsymbol\epsilon|\mathbf x_t]$.
Hence, for $t>1$:
\begin{align}
&\;\mathbb E_{q(\mathbf x_t,\mathbf x_0|c)}[D_{KL}(q_\delta(\mathbf x_{t-1}|\mathbf x_t,\mathbf x_0,c)\|\hat p_{\theta}(\mathbf x_{t-1}|\mathbf x_t,c))] \\
=&\;\mathbb E_{q(\mathbf x_t,\mathbf x_0|c)}[D_{KL}(q_\delta(\mathbf x_{t-1}|\mathbf x_t,\mathbf x_0,c)\|q_\delta(\mathbf x_{t-1}|\mathbf x_t,\hat{\mathbf f}_{\theta}^t(\mathbf x_t,c),c))] \\
\propto&\;\mathbb E_{q(\mathbf x_t,\mathbf x_0|c)}[\|\mathbf x_0-\hat{\mathbf f}_{\theta}^t(\mathbf x_t,c)\|_2^2] \\
\propto&\;\mathbb E_{\substack{\mathbf x_0\sim q(\mathbf x_0|c)\\\boldsymbol\epsilon\sim\mathcal N(\mathbf 0,\mathbf I)\\\mathbf x_t=\alpha_t\mathbf x_0+\sigma_t\boldsymbol\epsilon}}[\|\boldsymbol\epsilon-(\gamma\boldsymbol\epsilon_\theta(\mathbf x_t,c,t)+(1-\gamma)\boldsymbol\epsilon_\theta(\mathbf x_t,t))\|_2^2] \\
=&\;\mathbb E_{\mathbf x_0,\boldsymbol\epsilon}[\|\gamma(\boldsymbol\epsilon-\boldsymbol\epsilon_\theta(\mathbf x_t,c,t))+(1-\gamma)(\boldsymbol\epsilon-\boldsymbol\epsilon_\theta(\mathbf x_t,t))\|_2^2] \\
=&\;\mathbb E_{\mathbf x_0,\boldsymbol\epsilon}[\gamma^2\|\boldsymbol\epsilon-\boldsymbol\epsilon_\theta(\mathbf x_t,c,t)\|_2^2+(1-\gamma)^2\|\boldsymbol\epsilon-\boldsymbol\epsilon_\theta(\mathbf x_t,t)\|_2^2]\nonumber \\
&\;\qquad+2\gamma(1-\gamma)\mathbb E_{\mathbf x_0,\boldsymbol\epsilon}[\langle\boldsymbol\epsilon-\boldsymbol\epsilon_\theta(\mathbf x_t,c,t),\boldsymbol\epsilon-\boldsymbol\epsilon_\theta(\mathbf x_t,t)\rangle] \\
=&\;\mathbb E_{\mathbf x_0,\boldsymbol\epsilon}[\gamma^2\|\boldsymbol\epsilon-\boldsymbol\epsilon_\theta(\mathbf x_t,c,t)\|_2^2+(1-\gamma)^2\|\boldsymbol\epsilon-\boldsymbol\epsilon_\theta(\mathbf x_t,t)\|_2^2]\nonumber \\
&\;\qquad+2\gamma(1-\gamma)\mathbb E_{\mathbf x_0,\boldsymbol\epsilon}[\langle\boldsymbol\epsilon-\mathbb E_{q(\boldsymbol\epsilon|\mathbf x_t,c)}[\boldsymbol\epsilon|\mathbf x_t],\boldsymbol\epsilon-\boldsymbol\epsilon_\theta(\mathbf x_t,t)\rangle] \\
=&\;\mathbb E_{\mathbf x_0,\boldsymbol\epsilon}[\gamma^2\|\boldsymbol\epsilon-\boldsymbol\epsilon_\theta(\mathbf x_t,c,t)\|_2^2+(1-\gamma)^2\|\boldsymbol\epsilon-\boldsymbol\epsilon_\theta(\mathbf x_t,t)\|_2^2]\nonumber \\
&\;\qquad+2\gamma(1-\gamma)\mathbb E_{\mathbf x_0,\boldsymbol\epsilon}[\langle\boldsymbol\epsilon-\boldsymbol\epsilon,\boldsymbol\epsilon-\boldsymbol\epsilon_\theta(\mathbf x_t,t)\rangle]\label{eq:proof9} \\
=&\;\gamma^2\mathbb E_{\mathbf x_0,\boldsymbol\epsilon}[\|\boldsymbol\epsilon-\boldsymbol\epsilon_\theta(\mathbf x_t,c,t)\|_2^2]+(1-\gamma)^2\mathbb E_{\mathbf x_0,\boldsymbol\epsilon}[\|\boldsymbol\epsilon-\boldsymbol\epsilon_\theta(\mathbf x_t,t)\|_2^2],
\end{align}
in which \cref{eq:proof9} is from \Cref{lem:1}.
As for $t=1$ we have similar derivation:
\begin{align}
&\;\mathbb E_{q(\mathbf x_1,\mathbf x_0|c)}[-\log\hat p_{\theta}(\mathbf x_0|\mathbf x_1,c))] \\
\propto&\;\mathbb E_{q(\mathbf x_1,\mathbf x_0|c)}[\|\mathbf x_0-\hat{\mathbf f}_{\theta}^t(\mathbf x_1,c)\|_2^2]+C_2 \\
\propto&\;\mathbb E_{\substack{\mathbf x_0\sim q(\mathbf x_0|c)\\\boldsymbol\epsilon\sim\mathcal N(\mathbf 0,\mathbf I)\\\mathbf x_1=\alpha_1\mathbf x_0+\sigma_1\boldsymbol\epsilon}}[\|\boldsymbol\epsilon-(\gamma\boldsymbol\epsilon_\theta(\mathbf x_1,c,1)+(1-\gamma)\boldsymbol\epsilon_\theta(\mathbf x_1,1))\|_2^2]+C_3 \\
=&\;\gamma^2\mathbb E_{\mathbf x_0,\boldsymbol\epsilon}[\|\boldsymbol\epsilon-\boldsymbol\epsilon_\theta(\mathbf x_1,c,1)\|_2^2]+(1-\gamma)^2\mathbb E_{\mathbf x_0,\boldsymbol\epsilon}[\|\boldsymbol\epsilon-\boldsymbol\epsilon_\theta(\mathbf x_1,1)\|_2^2]+C_3,
\end{align}
in which $C_2$ and $C_3$ are constants not involving $\gamma$ and $\theta$.
Given that CFG involves score matching using both conditional and unconditional distributions, and that $J_{\delta,\gamma}$ is proportional to the score matching objective up to a constant, we confirm the equivalence between $J_{\delta,\gamma}$ and objective of native DPM under CFG.

Note that in native PF-ODE, we have
\begin{align}
\frac{\mathrm d\mathbf x_t}{\mathrm dt}&=f_t\mathbf x_t-\frac{1}{2}g^2_t\nabla_{\mathbf x_t}\log q_t(\mathbf x_t|c), \\
\mathbb E_{q_t(\mathbf x_t|c)}[\nabla_{\mathbf x_t}\log q_t(\mathbf x_t|c)]&=\mathbb E_{q_t(\mathbf x_t|c)}[\mathbb E_{q_t(\mathbf x_0|\mathbf x_t,c)}[\nabla_{\mathbf x_t}\log q_t(\mathbf x_t|\mathbf x_0,c)]] \\
&=\mathbb E_{q_t(\mathbf x_0,\mathbf x_t|c)}[\nabla_{\mathbf x_t}\log q_t(\mathbf x_t|\mathbf x_0,c)] \\
&=0\label{eq:proof13},
\end{align}
in which \cref{eq:proof13} holds since forward diffusion process $q_t(\mathbf x_t|\mathbf x_0,c)$ is implemented by adding Gaussian noise.
However, according to \cref{eq:pf_ode} and \Cref{lem:2}, we have
\begin{align}
\mathbb E_{\mathbf x_t}[s_{t,\gamma}(\mathbf x_t,c)]&=\mathbb E_{\mathbf x_t}[\gamma\nabla_{\mathbf x_t}\log q_t(\mathbf x_t|c)+(1-\gamma)\nabla_{\mathbf x_t}\log q_t(\mathbf x_t)] \\
&=(1-\gamma)\mathbb E_{\mathbf x_t}[\nabla_{\mathbf x_t}\log q_t(\mathbf x_t)] \\
&=\frac{\gamma-1}{\sigma_t^2}(\mathbb E_{\mathbf x_t}[\mathbf x_t]-\alpha_t\mathbb E_{c,\mathbf x_0,\mathbf x_t}[\mathbf x_0]) \\
&=\frac{\gamma-1}{\sigma_t^2}(\mathbb E_{q_t(\mathbf x_t|c)}[\mathbf x_t]-\alpha_t\mathbb E_{q_0(\mathbf x_0,c)}[\mathbf x_0]).
\end{align}
Note that $\mathbb E_{q_t(\mathbf x_t|c)}[\mathbf x_t]=\alpha_t\mathbb E_{q_0(\mathbf x_0|c)}[\mathbf x_0]$, and that $\mathbb E_{\mathbf x_0,c}[\mathbf x_0]=\int\mathbb E_{q_0(\mathbf x_0|c)}[\mathbf x_0]\mathrm dc$.
Therefore when $\gamma\neq1$, $\mathbb E_{\mathbf x_t}[s_{t,\gamma}(\mathbf x_t,c)]$ is not guaranteed to be identical with 0.
In other words, denoising with CFG cannot be expressed as a reciprocal of diffusion process with Gaussian noise.
\end{proof}

\subsection{Proof of \Cref{thm:gs_shift}}\label{subsec:proof.1}

\begin{proof}
Given \cref{eq:cfg}, for $\gamma>1$, we have
\begin{align}
s_{t,\gamma}(\mathbf x_t,c)&=\gamma\nabla_{\mathbf x_t}\log q_{t}(\mathbf x_t|c)+(1-\gamma)\nabla_{\mathbf x_t}\log q_{t}(\mathbf x_t) \\
&=-\gamma\frac{\mathbf x_t-c}{t+1}-(1-\gamma)\frac{\mathbf x_t}{t+2}, \\
\frac{\mathrm d\mathbf x_t}{\mathrm dt}&=-\frac{1}{2}s_{t,\gamma}(\mathbf x_t,c) \\
&=\mathbf x_t\left(\frac{\gamma}{2(t+1)}+\frac{1-\gamma}{2(t+2)}\right)-c\frac{\gamma}{2(t+1)}. \label{eq:proof1}
\end{align}
By variation of constants formula, we can analytically solve $q_{0,\gamma}^{\mathrm{deter}}(\mathbf x_0|c)$ in \cref{eq:proof1}.
\begin{align}
\mathbf x_t&=e^{\int_T^t\frac{\gamma}{2(s+1)}+\frac{1-\gamma}{2(s+2)}\mathrm ds}\left(C-\int_T^tc\frac{\gamma}{2(s+1)}e^{-\int_s^t\frac{\gamma}{2(r+1)}+\frac{1-\gamma}{2(r+2)}\mathrm dr}\mathrm ds\right) \\
&=(t+1)^{\frac{\gamma}{2}}(t+2)^{\frac{1-\gamma}{2}}\left(C-c\frac{\gamma}{2}\int_T^t(s+1)^{-\frac{\gamma+2}{2}}(s+2)^{-\frac{1-\gamma}{2}}\mathrm ds\right),
\end{align}
in which $C$ is a constant to determine.
Let $t=T$, we can see that
\begin{align}
C=\frac{\mathbf x_T}{(T+1)^{\frac{\gamma}{2}}(T+2)^{\frac{1-\gamma}{2}}}.
\end{align}
Therefore, we achieve the closed-form formula for $q_{0,\gamma}^{\mathrm{deter}}(\mathbf x_0|c)$ as below:
\begin{align}
\mathbf x_0=2^{\frac{1-\gamma}{2}}\left(\frac{\mathbf x_T}{(T+1)^{\frac{\gamma}{2}}(T+2)^{\frac{1-\gamma}{2}}}+c\frac{\gamma}{2}\int_0^T(s+1)^{-\frac{\gamma+2}{2}}(s+2)^{-\frac{1-\gamma}{2}}\mathrm ds\right).
\end{align}
Since $q_T(\mathbf x_T|c)\sim\mathcal N(c,T+1)$, we can deduce that
\begin{align}
q_{0,\gamma}^{\mathrm{deter}}(\mathbf x_0|c)\sim\mathcal N\left(c\phi(\gamma,T),2^{1-\gamma}\frac{T+1}{(T+1)^\gamma(T+2)^{1-\gamma}}\right),
\end{align}
in which
\begin{align}
\phi(\gamma,T)=2^{\frac{1-\gamma}{2}}\left(\frac{1}{(T+1)^{\frac{\gamma}{2}}(T+2)^{\frac{1-\gamma}{2}}}+\frac{\gamma}{2}\int_0^T(s+1)^{-\frac{\gamma+2}{2}}(s+2)^{-\frac{1-\gamma}{2}}\mathrm ds\right).
\end{align}

It is obvious that
\begin{align}
\phi(\gamma)&=2^{\frac{1-\gamma}{2}}\frac{\gamma}{2}\int_0^{+\infty}(s+1)^{-\frac{\gamma+2}{2}}(s+2)^{-\frac{1-\gamma}{2}}\mathrm ds, \\
&\lim_{T\rightarrow+\infty}\frac{T+1}{(T+1)^\gamma(T+2)^{1-\gamma}}=1.
\end{align}
Then it is suffices to calculate $\phi(\gamma)$ for all $\gamma>1$.
First note that
\begin{align}
\phi(1)&=\frac{1}{2}\int_0^{+\infty}(s+1)^{-\frac{3}{2}}\mathrm ds=1,\label{eq:proof4} \\
\phi(3)&=2^{-1}\frac{3}{2}\int_0^{+\infty}(s+1)^{-\frac{5}{2}}(s+2)\mathrm ds=2,\label{eq:proof5} \\
\phi(5)&=2^{-2}\frac{5}{2}\int_0^{+\infty}(s+1)^{-\frac{7}{2}}(s+2)^2\mathrm ds=\frac{7}{3}\label{eq:proof6}.
\end{align}
For $\gamma>1$, denote by $I(\gamma)$ with
\begin{align}
I(\gamma)=\int_0^{+\infty}(s+1)^{-\frac{\gamma+2}{2}}(s+2)^{-\frac{1-\gamma}{2}}\mathrm ds.
\end{align}
Note that for $\gamma>1$ we have
\begin{align}
I(\gamma)&=\int_0^{+\infty}(s+1)^{-\frac{\gamma+2}{2}}(s+2)^{-\frac{1-\gamma}{2}}\mathrm ds \\
&=\frac{2}{\gamma+1}\int_0^{+\infty}(s+1)^{-\frac{\gamma+2}{2}}\mathrm d(s+2)^{\frac{\gamma+1}{2}} \\
&=\frac{2}{\gamma+1}\left((s+1)^{-\frac{\gamma+1}{2}}(s+2)^{\frac{\gamma+1}{2}}\Big|_0^{+\infty}+\frac{\gamma+2}{2}\int_0^{+\infty}(s+1)^{-\frac{\gamma+4}{2}}(s+2)^{\frac{\gamma+1}{2}}\mathrm ds\right) \\
&=\frac{2}{\gamma+1}\left(\frac{\gamma+2}{2}I(\gamma+2)-2^{\frac{\gamma+1}{2}}\right).
\end{align}
Therefore, for $\gamma>1$ we have
\begin{align}
\phi(\gamma)&=\frac{2\gamma}{\gamma+1}(\phi(\gamma+2)-1), \\
\phi(\gamma+2)&=1+\frac{\gamma+1}{2\gamma}\phi(\gamma)\label{eq:proof_step}.
\end{align}
From \cref{eq:proof4,eq:proof5,eq:proof6} we have
\begin{align}
I(1)=2,\quad I(3)=\frac{8}{3},\quad I(5)=\frac{56}{15}.
\end{align}
For $\gamma\in[1,3]$, by Cauchy-Schwarz inequality with $p\in[0,1]$, we have
\begin{align}
&\Big(I(\gamma)\Big)^p\Big(I(5)\Big)^{1-p} \\
=&\left(\int_0^{+\infty}(s+1)^{-\frac{\gamma+2}{2}}(s+2)^{-\frac{1-\gamma}{2}}\mathrm ds\right)^p\left(\int_0^{+\infty}(s+1)^{-\frac{7}{2}}(s+2)^2\mathrm ds\right)^{1-p} \\
\geqslant&\int_0^{+\infty}\Big((s+1)^{-\frac{\gamma+2}{2}}(s+2)^{-\frac{1-\gamma}{2}}\Big)^p\Big((s+1)^{-\frac{7}{2}}(s+2)^2\Big)^{1-p}\mathrm ds \\
=&\int_0^{+\infty}(s+1)^{-\frac{\gamma p-5p+7}{2}}(s+2)^{-\frac{5p-\gamma p-4}{2}}\mathrm ds\label{eq:proof7}.
\end{align}
Let $p=\frac{2}{5-\gamma}\in[0,1]$ for $\gamma\in[1,3]$, from \cref{eq:proof7} we have
\begin{align}
I(\gamma)\geqslant\Big(I(3)\Big)^{\frac{5-\gamma}{2}}\Big(I(5)\Big)^{\frac{\gamma-3}{2}}=\left(\frac{8}{3}\right)^{\frac{5-\gamma}{2}}\left(\frac{56}{15}\right)^{\frac{\gamma-3}{2}}.
\end{align}
Therefore for $\gamma\in[1,3]$, we have
\begin{align}
\phi(\gamma)\geqslant2^{\frac{1-\gamma}{2}}\frac{\gamma}{2}\left(\frac{8}{3}\right)^{\frac{5-\gamma}{2}}\left(\frac{56}{15}\right)^{\frac{\gamma-3}{2}}=\gamma\frac{7}{15}\left(\frac{10}{7}\right)^{\frac{5-\gamma}{2}}=:h_1(\gamma)
\end{align}
Since $\frac{1}{\gamma}-\frac{1}{2}\log\frac{10}{7}>0$ for $\gamma\in[1,3]$, $h_1(\gamma)$ increases monotonically on $\in[1,3]$ and $h_1(1)=\frac{20}{21}$, $h_1(3)=2$.
Similarly, for $\gamma\in[3,5]$, by Cauchy-Schwarz inequality with $p\in[0,1]$, we have
\begin{align}
&\Big(I(1)\Big)^{1-p}\Big(I(\gamma)\Big)^p \\
=&\left(\int_0^{+\infty}(s+1)^{-\frac{3}{2}}\mathrm ds\right)^{1-p}\left(\int_0^{+\infty}(s+1)^{-\frac{\gamma+2}{2}}(s+2)^{-\frac{1-\gamma}{2}}\mathrm ds\right)^p \\
\geqslant&\int_0^{+\infty}\Big((s+1)^{-\frac{3}{2}}\Big)^{1-p}\Big((s+1)^{-\frac{\gamma+2}{2}}(s+2)^{-\frac{1-\gamma}{2}}\Big)^p\mathrm ds \\
=&\int_0^{+\infty}(s+1)^{-\frac{3-p+\gamma p}{2}}(s+2)^{-\frac{p-\gamma p}{2}}\mathrm ds\label{eq:proof8}.
\end{align}
Let $p=\frac{2}{\gamma-1}\in[0,1]$ for $\gamma\in[3,5]$, from \cref{eq:proof8} we have
\begin{align}
I(\gamma)\geqslant\Big(I(1)\Big)^{\frac{3-\gamma}{2}}\Big(I(3)\Big)^{\frac{\gamma-1}{2}}=2^{\frac{3-\gamma}{2}}\left(\frac{8}{3}\right)^{\frac{\gamma-1}{2}}.
\end{align}
Therefore for $\gamma\in[3,5]$
\begin{align}
\phi(\gamma)\geqslant2^{\frac{1-\gamma}{2}}\frac{\gamma}{2}2^{\frac{3-\gamma}{2}}\left(\frac{8}{3}\right)^{\frac{\gamma-1}{2}}=\gamma\left(\frac{2}{3}\right)^{\frac{\gamma-1}{2}}=:h_2(\gamma)
\end{align}
It is easy to see that $h_2(\gamma)\geqslant2$ for $\gamma\in[3,5]$.
Then by mathematical induction and \cref{eq:proof_step}, we have $\phi(\gamma)\geqslant2$ for all $\gamma\geqslant3$.
Specially, we have
\begin{align}
\lim_{\gamma\rightarrow+\infty}\phi(\gamma)=2.
\end{align}

And specifically, for $\gamma\in\mathbb N$, $\gamma>1$, we analytically calculate $\phi(\gamma,T)$ for $\gamma=2n+1$ and $\gamma=2n$, respectively.
First let $\gamma=2n+1$, $n\in\mathbb N$.
We can see that
\begin{align}
(s+2)^{-\frac{1-\gamma}{2}}=(s+2)^n=\sum_{k=0}^nC_n^k(s+1)^k.
\end{align}
\begin{align}
\int_0^T(s+1)^{-\frac{\gamma+2}{2}}(s+2)^{-\frac{1-\gamma}{2}}\mathrm ds&=\int_0^T(s+1)^{-\frac{\gamma+2}{2}}\left(\sum_{k=0}^nC_n^k(s+1)^k\right)\mathrm ds \\
&=\sum_{k=0}^n\left(C_n^k\int_0^T(s+1)^{\frac{2k-\gamma-2}{2}}\mathrm ds\right) \\
&=\sum_{k=0}^n\left(C_n^k\frac{2}{2k-\gamma}\Big((T+1)^{\frac{2k-\gamma}{2}}-1\Big)\right).
\end{align}
Since $2k-\gamma<0$ for $k=0,1,\cdots,n$, we have $(T+1)^{\frac{2k-\gamma}{2}}\rightarrow0$ as $T$ goes to infinity and hence
\begin{align}
&\;\phi(\gamma,T) \\
=&\;2^{\frac{1-\gamma}{2}}\left(\frac{1}{(T+1)^{\frac{\gamma}{2}}(T+2)^{\frac{1-\gamma}{2}}}+\frac{\gamma}{2}\int_0^T(s+1)^{-\frac{\gamma+2}{2}}(s+2)^{-\frac{1-\gamma}{2}}\mathrm ds\right) \\
=&\;2^{\frac{1-\gamma}{2}}\left(\frac{1}{(T+1)^{\frac{\gamma}{2}}(T+2)^{\frac{1-\gamma}{2}}}+\frac{\gamma}{2}\left(\sum_{k=0}^nC_n^k\frac{2}{2k-\gamma}\Big((T+1)^{\frac{2k-\gamma}{2}}-1\Big)\right)\right) \\
=&\;2^{\frac{1-\gamma}{2}}\left(\frac{1}{(T+1)^{\frac{\gamma}{2}}(T+2)^{\frac{1-\gamma}{2}}}+\left(\sum_{k=0}^nC_n^k\frac{2n+1}{2n-2k+1}\Big(1-(T+1)^{\frac{2k-\gamma}{2}}\Big)\right)\right).
\end{align}
When $T\rightarrow+\infty$, we have
\begin{align}
\phi(2n+1)=2^{-n}\left(\sum_{k=0}^nC_n^k\frac{2n+1}{2n-2k+1}\right).
\end{align}

Then let $\gamma=2n$, $n\in\mathbb N$, and $n\geqslant1$.
We have
\begin{align}
&\;\int(s+1)^{-\frac{\gamma+2}{2}}(s+2)^{-\frac{1-\gamma}{2}}\mathrm ds \\
=&\;\int2(s+1)^{-n-1}(\sqrt{s+2})^{2n}\mathrm d\sqrt{s+2} \\
=&\;\int2(u^2-1)^{-n-1}u^{2n}\mathrm du\label{eq:proof2} \\
=&\;-\frac{1}{n}u^{2n-1}(u^2-1)^{-n}+\frac{2n-1}{n}\int(u^2-1)^{-n}u^{2n-2}\mathrm du\label{eq:proof3},
\end{align}
in which \cref{eq:proof2} is due to integration by substitution with $u=\sqrt{s+2}>1$, and \cref{eq:proof3} is due to integration by parts.
Denote by $I_n$ with
\begin{align}
I_n=\int2(u^2-1)^{-n-1}u^{2n}\mathrm du,\quad n\geqslant0,
\end{align}
then we have
\begin{align}
I_n=-\frac{1}{n}u^{2n-1}(u^2-1)^{-n}+\frac{2n-1}{2n}I_{n-1},\quad n\geqslant1.
\end{align}
For $n\geqslant1$, let $I_n=\frac{(2n-1)!!}{(2n)!!}A_n$, then we have
\begin{align}
\frac{(2n-1)!!}{(2n)!!}A_n&=-\frac{1}{n}u^{2n-1}(u^2-1)^{-n}+\frac{2n-1}{2n}\frac{(2n-3)!!}{(2n-2)!!}A_{n-1},\quad n\geqslant2, \\
A_n&=A_{n-1}-\frac{1}{n}\frac{(2n)!!}{(2n-1)!!}u^{2n-1}(u^2-1)^{-n},\quad n\geqslant2.
\end{align}
Therefore for $n\geqslant2$, we have
\begin{align}
A_n=A_1-\sum_{k=2}^n\frac{1}{k}\frac{(2k)!!}{(2k-1)!!}u^{2k-1}(u^2-1)^{-k},
\end{align}
and
\begin{align}
I_n=
\begin{cases}
\mathlarger{-\frac{u}{u^2-1}+\frac{1}{2}\log\frac{u-1}{u+1}}, & n=1, \\[20pt]
\mathlarger{\frac{(2n-1)!!}{(2n)!!}\left(-\frac{2u}{u^2-1}+\log\frac{u-1}{u+1}-\sum_{k=2}^n\frac{1}{k}\frac{(2k)!!}{(2k-1)!!}\frac{u^{2k-1}}{(u^2-1)^{k}}\right)}, & n\geqslant2.
\end{cases}
\end{align}
Therefore, for $\gamma=2$ we have
\begin{align}
\phi(\gamma,T)&=2^{\frac{1-\gamma}{2}}\left(\frac{1}{(T+1)^{\frac{\gamma}{2}}(T+2)^{\frac{1-\gamma}{2}}}+\frac{\gamma}{2}\int_0^T(s+1)^{-\frac{\gamma+2}{2}}(s+2)^{-\frac{1-\gamma}{2}}\mathrm ds\right) \\
&=2^{\frac{1-\gamma}{2}}\frac{1}{(T+1)^{\frac{\gamma}{2}}(T+2)^{\frac{1-\gamma}{2}}}\nonumber \\
&\qquad-2^{\frac{1-\gamma}{2}}\frac{\gamma}{2}\left(\frac{\sqrt{T+2}}{T+1}-\sqrt2\right)\nonumber \\
&\qquad\qquad+2^{\frac{1-\gamma}{2}}\frac{\gamma}{2}\frac{1}{2}\left(\log\frac{\sqrt{T+2}-1}{\sqrt{T+2}+1}-\log\frac{\sqrt2-1}{\sqrt2+1}\right),
\end{align}
and for $\gamma\geqslant4$ we have
\begin{align}
&\;\phi(\gamma,T) \\
=&\;2^{\frac{1-\gamma}{2}}\left(\frac{1}{(T+1)^{\frac{\gamma}{2}}(T+2)^{\frac{1-\gamma}{2}}}+\frac{\gamma}{2}\int_0^T(s+1)^{-\frac{\gamma+2}{2}}(s+2)^{-\frac{1-\gamma}{2}}\mathrm ds\right) \\
=&\;2^{\frac{1-\gamma}{2}}\frac{1}{(T+1)^{\frac{\gamma}{2}}(T+2)^{\frac{1-\gamma}{2}}}\nonumber \\
&\;\qquad-2^{\frac{1-\gamma}{2}}\frac{\gamma}{2}\frac{(2n-1)!!}{(2n)!!}\left(\frac{2\sqrt{T+2}}{T+1}-2\sqrt2\right)\nonumber \\
&\;\qquad\qquad+2^{\frac{1-\gamma}{2}}\frac{\gamma}{2}\frac{(2n-1)!!}{(2n)!!}\left(\log\frac{\sqrt{T+2}-1}{\sqrt{T+2}+1}-\log\frac{\sqrt2-1}{\sqrt2+1}\right)\nonumber \\
&\;\qquad\qquad\qquad-2^{\frac{1-\gamma}{2}}\frac{\gamma}{2}\frac{(2n-1)!!}{(2n)!!}\left(\sum_{k=2}^n\frac{1}{k}\frac{(2k)!!}{(2k-1)!!}\Big(\frac{(T+2)^{\frac{2k-1}{2}}}{(T+1)^{k}}-2^{\frac{2k-1}{2}}\Big)\right).
\end{align}
When $T\rightarrow+\infty$, we have
\begin{align}
\phi(2n)=
\begin{cases}
\mathlarger{2^{\frac{1}{2}-n}n\left(\sqrt2-\frac{1}{2}\log\frac{\sqrt2-1}{\sqrt2+1}\right)}, & n=1, \\[20pt]
\mathlarger{\frac{(2n-1)!!\sqrt2n}{(2n)!!2^n}\left(\left(\sum_{k=2}^n\frac{1}{k}\frac{(2k)!!}{(2k-1)!!}2^{k-\frac{1}{2}}\right)+2\sqrt2-\log\frac{\sqrt2-1}{\sqrt2+1}\right)}, & n\geqslant2.
\end{cases}
\end{align}
\end{proof}

\subsection{Proof of \Cref{thm:generalized_cfg_expectation}}\label{subsec:proof.3}

\begin{proof}
We first write the closed-form expressions of DDIM sampler as below:
\begin{align}
\mathbf x_{t-1}&=\frac{\alpha_{t-1}}{\alpha_t}\mathbf x_t+(\sigma_{t-1}-\frac{\alpha_{t-1}}{\alpha_t}\sigma_t)\boldsymbol\epsilon_\theta(\mathbf x_t,c,t), \\
\tilde{\mathbf x}_{t-1}&=\frac{\alpha_{t-1}}{\alpha_t}\tilde{\mathbf x}_t+(\sigma_{t-1}-\frac{\alpha_{t-1}}{\alpha_t}\sigma_t)(\gamma_1\boldsymbol\epsilon_\theta(\tilde{\mathbf x}_t,c,t)+\gamma_0\boldsymbol\epsilon_\theta(\tilde{\mathbf x}_t,t)).
\end{align}
Then we have
\begin{align}
\Delta_{t-1}&=\mathbb E_{\mathbf x_t}[\mathbf x_{t-1}]-\mathbb E_{\tilde{\mathbf x}_t}[\tilde{\mathbf x}_{t-1}] \\
&=\frac{\alpha_{t-1}}{\alpha_t}(\mathbb E_{\mathbf x_t}[\mathbf x_t]-\mathbb E_{\tilde{\mathbf x}_t}[\tilde{\mathbf x}_t])\nonumber \\
&\qquad+(\sigma_{t-1}-\frac{\alpha_{t-1}}{\alpha_t}\sigma_t)(\mathbb E_{\mathbf x_t}[\boldsymbol\epsilon_\theta(\mathbf x_t,c)]-\mathbb E_{\tilde{\mathbf x}_t}[\gamma_1\boldsymbol\epsilon_\theta(\tilde{\mathbf x}_t,c,t)+\gamma_0\boldsymbol\epsilon_\theta(\tilde{\mathbf x}_t,t)]).
\end{align}
Note that
\begin{align}
\boldsymbol\epsilon_\theta(\mathbf x_t,c,t)=\mathbb E_{q(\mathbf x_0|\mathbf x_t,c)}\left[\frac{\mathbf x_t-\alpha_t\mathbf x_0}{\sigma_t}|\mathbf x_t\right],\quad\boldsymbol\epsilon_\theta(\tilde{\mathbf x}_t,c,t)=\mathbb E_{q(\mathbf x_0|\tilde{\mathbf x}_t,c)}\left[\frac{\tilde{\mathbf x}_t-\alpha_t\mathbf x_0}{\sigma_t}|\tilde{\mathbf x}_t\right].
\end{align}
Therefore, by $q_t(\mathbf x_t|c)=\int q_0(\mathbf x_0|c)q_{0t}(\mathbf x_t|\mathbf x_0,c)\mathrm d\mathbf x_0$ and \Cref{lem:1} we have
\begin{align}\label{eq:cond_exp_gt}
\mathbb E_{\mathbf x_t}[\boldsymbol\epsilon_\theta(\mathbf x_t,c,t)]=\mathbb E_{\mathbf x_0,\mathbf x_t}\left[\frac{\mathbf x_t-\alpha_t\mathbf x_0}{\sigma_t}\right]=\frac{1}{\sigma_t}\mathbb E_{\mathbf x_t}[\mathbf x_t]-\frac{\alpha_t}{\sigma_t}\mathbb E_{\mathbf x_0}[\mathbf x_0].
\end{align}
Similarly, by $\hat p_{\theta}(\tilde{\mathbf x}_t|c)=\int q_0(\mathbf x_0|c)q_{0T}(\mathbf x_T|\mathbf x_0,c)\hat p_{\theta}(\tilde{\mathbf x}_t|\mathbf x_T,c)\mathrm d\mathbf x_0\mathrm d\mathbf x_T$ we have
\begin{align}\label{eq:cond_exp_ddim}
\mathbb E_{\tilde{\mathbf x}_t}[\boldsymbol\epsilon_\theta(\tilde{\mathbf x}_t,c,t)]=\mathbb E_{\mathbf x_0,\tilde{\mathbf x}_t}\left[\frac{\tilde{\mathbf x}_t-\alpha_t\mathbf x_0}{\sigma_t}\right]=\frac{1}{\sigma_t}\mathbb E_{\tilde{\mathbf x}_t}[\tilde{\mathbf x}_t]-\frac{\alpha_t}{\sigma_t}\mathbb E_{\mathbf x_0}[\mathbf x_0].
\end{align}
Then we can simplify $\Delta_t$ as below:
\begin{align}
\Delta_{t-1}&=\frac{\alpha_{t-1}}{\alpha_t}\Delta_t+(\sigma_{t-1}-\frac{\alpha_{t-1}}{\alpha_t}\sigma_t)(\frac{1}{\sigma_t}\Delta_t-\mathbb E_{\tilde{\mathbf x}_t}[(\gamma_1-1)\boldsymbol\epsilon_\theta(\tilde{\mathbf x}_t,c,t)+\gamma_0\boldsymbol\epsilon_\theta(\tilde{\mathbf x}_t,t)]) \\
&=\frac{\sigma_{t-1}}{\sigma_t}\Delta_t-(\sigma_{t-1}-\frac{\alpha_{t-1}}{\alpha_t}\sigma_t)\mathbb E_{\tilde{\mathbf x}_t}[(\gamma_1-1)\boldsymbol\epsilon_\theta(\tilde{\mathbf x}_t,c,t)+\gamma_0\boldsymbol\epsilon_\theta(\tilde{\mathbf x}_t,t)]) \\
&=\frac{\sigma_{t-1}}{\sigma_t}\Delta_t-(\sigma_{t-1}-\frac{\alpha_{t-1}}{\alpha_t}\sigma_t)\mathbb E_{\tilde{\mathbf x}_t}[\boldsymbol\epsilon_{\gamma_1,\gamma_0}(\tilde{\mathbf x}_t)]).
\end{align}
$\Delta_t=0$ implies that $\mathbb E_{q_t(\mathbf x_t|c)}[\mathbf x_t]=\mathbb E_{\hat p_{\theta}(\tilde{\mathbf x}_t|c)}[\tilde{\mathbf x}_t]$.
Therefore, by \cref{eq:cond_exp_gt,eq:cond_exp_ddim} we have $\mathbb E_{\mathbf x_t}[\boldsymbol\epsilon_\theta(\mathbf x_t,c,t)]=\mathbb E_{\tilde{\mathbf x}_t}[\boldsymbol\epsilon_\theta(\tilde{\mathbf x}_t,c,t)]$.
According to \Cref{lem:2} and by calculating the expectation over $\mathbf x_t$ and $\tilde{\mathbf x}_t$ respectively, we have
\begin{align}
\mathbb E_{\tilde{\mathbf x}_t}[\boldsymbol\epsilon_\theta(\tilde{\mathbf x}_t,t)]&=\frac{1}{\sigma_t}\mathbb E_{\tilde{\mathbf x}_t}[\tilde{\mathbf x}_t]-\frac{\alpha_t}{\sigma_t}\mathbb E_{c,\mathbf x_0,\tilde{\mathbf x}_t}[\mathbf x_0]=\frac{1}{\sigma_t}\mathbb E_{\tilde{\mathbf x}_t}[\tilde{\mathbf x}_t]-\frac{\alpha_t}{\sigma_t}\mathbb E_{c,\mathbf x_0}[\mathbf x_0], \\
\mathbb E_{\mathbf x_t}[\boldsymbol\epsilon_\theta(\mathbf x_t,t)]&=\frac{1}{\sigma_t}\mathbb E_{\tilde{\mathbf x}_t}[\mathbf x_t]-\frac{\alpha_t}{\sigma_t}\mathbb E_{c,\mathbf x_0,\mathbf x_t}[\mathbf x_0]=\frac{1}{\sigma_t}\mathbb E_{\mathbf x_t}[\mathbf x_t]-\frac{\alpha_t}{\sigma_t}\mathbb E_{c,\mathbf x_0}[\mathbf x_0].
\end{align}
Since $\Delta_t=0$, we have $\mathbb E_{\mathbf x_t}[\boldsymbol\epsilon_\theta(\mathbf x_t,t)]=\mathbb E_{\tilde{\mathbf x}_t}[\boldsymbol\epsilon_\theta(\tilde{\mathbf x}_t,t)]$, and thus
\begin{align}
\Delta_{t-1}&=-(\sigma_{t-1}-\frac{\alpha_{t-1}}{\alpha_t}\sigma_t)\mathbb E_{\mathbf x_t}[(\gamma_1-1)\boldsymbol\epsilon_\theta(\mathbf x_t,c,t)+\gamma_0\boldsymbol\epsilon_\theta(\mathbf x_t,t)]) \\
&=-(\sigma_{t-1}-\frac{\alpha_{t-1}}{\alpha_t}\sigma_t)\mathbb E_{\mathbf x_t}[\boldsymbol\epsilon_{\gamma_1,\gamma_0}(\mathbf x_t)].
\end{align}
\end{proof}

\subsection{Proof of \Cref{thm:generalized_cfg_variance}}\label{subsec:proof.4}

\begin{proof}
Given \cref{eq:generalized_cfg}, for any $\gamma_1$ and $\gamma_0$, we have
\begin{align}
s_{t,\gamma_1,\gamma_0}(\mathbf x_t,c)&=\gamma_1\nabla_{\mathbf x_t}\log q_{t}(\mathbf x_t|c)+\gamma_0\nabla_{\mathbf x_t}\log q_{t}(\mathbf x_t) \\
&=-\gamma_1\frac{\mathbf x_t-c}{t+1}-\gamma_0\frac{\mathbf x_t}{t+2}, \\
\frac{\mathrm d\mathbf x_t}{\mathrm dt}&=-\frac{1}{2}s_{t,\gamma_1,\gamma_0}(\mathbf x_t,c) \\
&=\mathbf x_t\left(\frac{\gamma_1}{2(t+1)}+\frac{\gamma_0}{2(t+2)}\right)-c\frac{\gamma_1}{2(t+1)}. \label{eq:proof11}
\end{align}
By variation of constants formula, we can analytically solve $q_{0,\gamma_1,\gamma_0}^{\mathrm{deter}}(\mathbf x_0|c)$ in \cref{eq:proof11}.
\begin{align}
\mathbf x_t&=e^{\int_T^t\frac{\gamma_1}{2(s+1)}+\frac{\gamma_0}{2(s+2)}\mathrm ds}\left(C-\int_T^tc\frac{\gamma_1}{2(s+1)}e^{-\int_s^t\frac{\gamma_1}{2(r+1)}+\frac{\gamma_0}{2(r+2)}\mathrm dr}\mathrm ds\right) \\
&=(t+1)^{\frac{\gamma_1}{2}}(t+2)^{\frac{\gamma_0}{2}}\left(C-c\frac{\gamma_1}{2}\int_T^t(s+1)^{-\frac{\gamma_1+2}{2}}(s+2)^{-\frac{\gamma_0}{2}}\mathrm ds\right),
\end{align}
in which $C$ is a constant to determine.
Let $t=T$, we can see that
\begin{align}
C=\frac{\mathbf x_T}{(T+1)^{\frac{\gamma_1}{2}}(T+2)^{\frac{\gamma_0}{2}}}.
\end{align}
Therefore, we achieve the closed-form formula for $q_{0,\gamma}^{\mathrm{deter}}(\mathbf x_0|c)$ as below:
\begin{align}
\mathbf x_0=2^{\frac{\gamma_0}{2}}\left(\frac{\mathbf x_T}{(T+1)^{\frac{\gamma_1}{2}}(T+2)^{\frac{\gamma_0}{2}}}+c\frac{\gamma_1}{2}\int_0^T(s+1)^{-\frac{\gamma_1+2}{2}}(s+2)^{-\frac{\gamma_0}{2}}\mathrm ds\right).
\end{align}
Since $q_T(\mathbf x_T|c)\sim\mathcal N(c,T+1)$, we can deduce that
\begin{align}
\mathrm{var}_{q_{0,\gamma_1,\gamma_0}^{\mathrm{deter}}(\mathbf x_0|c)}[\mathbf x_0]=2^{\gamma_0}(T+1)^{1-\gamma_1}(T+2)^{-\gamma_0}.
\end{align}
\end{proof}

\subsection{Proof of \Cref{thm:generalized_cfg}}\label{subsec:proof.5}

\begin{proof}
According to \cref{eq:generalized_cfg_ddim1,eq:generalized_cfg_ddim2}, we can write the variational lower bound of $\hat p_{\theta}(\mathbf x_{0:T}|c)$ as below:
\begin{align}
J_{\delta,\gamma_1,\gamma_0}&=\mathbb E_{q_\delta(\mathbf x_{0:T}|c)}[\log q_\delta(\mathbf x_{1:T}|\mathbf x_0,c)-\log\hat p_{\theta}(\mathbf x_{0:T}|c)] \\
&=\mathbb E\left[-\log\hat p_{\theta}(\mathbf x_0|\mathbf x_1,c)\right]\nonumber \\
&\qquad+\mathbb E\left[\sum_{t=2}^TD_{KL}(q_\delta(\mathbf x_{t-1}|\mathbf x_t,\mathbf x_0,c)\|\hat p_{\theta}(\mathbf x_{t-1}|\mathbf x_t,c))\right]\nonumber \\
&\qquad\qquad+C_1,
\end{align}
in which $C_1$ is a constant not involving $\gamma_1$, $\gamma_0$, and $\theta$.

Note that $\boldsymbol\epsilon_\theta(\mathbf x_t,c,t)=\mathbb E_{q(\boldsymbol\epsilon|\mathbf x_t,c)}[\boldsymbol\epsilon|\mathbf x_t]$.
Hence, for $t>1$:
\begin{align}
&\;\mathbb E_{q(\mathbf x_t,\mathbf x_0|c)}[D_{KL}(q_\delta(\mathbf x_{t-1}|\mathbf x_t,\mathbf x_0,c)\|\hat p_{\theta}(\mathbf x_{t-1}|\mathbf x_t,c))] \\
=&\;\mathbb E_{q(\mathbf x_t,\mathbf x_0|c)}[D_{KL}(q_\delta(\mathbf x_{t-1}|\mathbf x_t,\mathbf x_0,c)\|q_\delta(\mathbf x_{t-1}|\mathbf x_t,\hat{\mathbf f}_{\theta}^t(\mathbf x_t,c),c))] \\
\propto&\;\mathbb E_{q(\mathbf x_t,\mathbf x_0|c)}[\|\mathbf x_0-\hat{\mathbf f}_{\theta}^t(\mathbf x_t,c)\|_2^2] \\
\propto&\;\mathbb E_{\substack{\mathbf x_0\sim q(\mathbf x_0|c)\\\boldsymbol\epsilon\sim\mathcal N(\mathbf 0,\mathbf I)\\\mathbf x_t=\alpha_t\mathbf x_0+\sigma_t\boldsymbol\epsilon}}[\|\boldsymbol\epsilon-(\gamma_1\boldsymbol\epsilon_\theta(\mathbf x_t,c,t)+\gamma_0\boldsymbol\epsilon_\theta(\mathbf x_t,t))\|_2^2] \\
=&\;\mathbb E_{\mathbf x_0,\boldsymbol\epsilon}[\|\boldsymbol\epsilon\|_2^2+\|\gamma_1\boldsymbol\epsilon_\theta(\mathbf x_t,c,t)+\gamma_0\boldsymbol\epsilon_\theta(\mathbf x_t,t)\|_2^2]\nonumber \\
&\;\qquad-2\mathbb E_{\mathbf x_0,\boldsymbol\epsilon}[\langle\boldsymbol\epsilon,\gamma_1\boldsymbol\epsilon_\theta(\mathbf x_t,c,t)+\gamma_0\boldsymbol\epsilon_\theta(\mathbf x_t,t)\rangle] \\
=&\;\mathbb E_{\mathbf x_0,\boldsymbol\epsilon}[\|\boldsymbol\epsilon\|_2^2+\|\gamma_1\boldsymbol\epsilon_\theta(\mathbf x_t,c,t)+\gamma_0\boldsymbol\epsilon_\theta(\mathbf x_t,t)\|_2^2]\nonumber \\
&\;\qquad-2\mathbb E_{\mathbf x_0,\boldsymbol\epsilon}[\langle\mathbb E_{q(\boldsymbol\epsilon|\mathbf x_t,c)}[\boldsymbol\epsilon|\mathbf x_t],\gamma_1\boldsymbol\epsilon_\theta(\mathbf x_t,c,t)+\gamma_0\boldsymbol\epsilon_\theta(\mathbf x_t,t)\rangle]\label{eq:proof12} \\
=&\;\mathbb E_{\mathbf x_0,\boldsymbol\epsilon}[\|\boldsymbol\epsilon\|_2^2+\|\gamma_1\boldsymbol\epsilon_\theta(\mathbf x_t,c,t)+\gamma_0\boldsymbol\epsilon_\theta(\mathbf x_t,t)\|_2^2]\nonumber \\
&\;\qquad-2\mathbb E_{\mathbf x_0,\boldsymbol\epsilon}[\langle\boldsymbol\epsilon_\theta(\mathbf x_t,c,t),\gamma_1\boldsymbol\epsilon_\theta(\mathbf x_t,c,t)+\gamma_0\boldsymbol\epsilon_\theta(\mathbf x_t,t)\rangle] \\
=&\;\mathbb E_{\mathbf x_0,\boldsymbol\epsilon}[\|\boldsymbol\epsilon_\theta(\mathbf x_t,c,t)\|_2^2+\|\gamma_1\boldsymbol\epsilon_\theta(\mathbf x_t,c,t)+\gamma_0\boldsymbol\epsilon_\theta(\mathbf x_t,t)\|_2^2]\nonumber \\
&\;\qquad-2\mathbb E_{\mathbf x_0,\boldsymbol\epsilon}[\langle\boldsymbol\epsilon_\theta(\mathbf x_t,c,t),\gamma_1\boldsymbol\epsilon_\theta(\mathbf x_t,c,t)+\gamma_0\boldsymbol\epsilon_\theta(\mathbf x_t)\rangle]\nonumber \\
&\;\qquad\qquad+\mathbb E_{\mathbf x_0,\boldsymbol\epsilon}[\|\boldsymbol\epsilon\|_2^2-\|\boldsymbol\epsilon_\theta(\mathbf x_t,c,t)\|_2^2] \\
=&\;\mathbb E_{\mathbf x_0,\boldsymbol\epsilon}[\|\boldsymbol\epsilon_\theta(\mathbf x_t,c,t)-(\gamma_1\boldsymbol\epsilon_\theta(\mathbf x_t,c,t)+\gamma_0\boldsymbol\epsilon_\theta(\mathbf x_t,t))\|_2^2]+C_2 \\
=&\;\mathbb E_{\mathbf x_0,\boldsymbol\epsilon}[\|(\gamma_1-1)\boldsymbol\epsilon_\theta(\mathbf x_t,c,t)+\gamma_0\boldsymbol\epsilon_\theta(\mathbf x_t,t)\|_2^2]+C_2,
\end{align}
in which \cref{eq:proof12} is from \Cref{lem:1}, and $C_2=\mathbb E_{\mathbf x_0,\boldsymbol\epsilon}[\|\boldsymbol\epsilon\|_2^2-\|\boldsymbol\epsilon_\theta(\mathbf x_t,c,t)\|_2^2]$ is constant not involving $\gamma_1$ and $\gamma_0$.
As for $t=1$ we have similar derivation:
\begin{align}
&\;\mathbb E_{q(\mathbf x_1,\mathbf x_0|c)}[-\log\hat p_{\theta}(\mathbf x_0|\mathbf x_1,c))] \\
\propto&\;\mathbb E_{q(\mathbf x_1,\mathbf x_0|c)}[\|\mathbf x_0-\hat{\mathbf f}_{\theta}^t(\mathbf x_1,c)\|_2^2]+C_3 \\
\propto&\;\mathbb E_{\substack{\mathbf x_0\sim q(\mathbf x_0|c)\\\boldsymbol\epsilon\sim\mathcal N(\mathbf 0,\mathbf I)\\\mathbf x_1=\alpha_1\mathbf x_0+\sigma_1\boldsymbol\epsilon}}[\|\boldsymbol\epsilon-(\gamma_1\boldsymbol\epsilon_\theta(\mathbf x_1,c,1)+\gamma_0\boldsymbol\epsilon_\theta(\mathbf x_1,1))\|_2^2]+C_4 \\
=&\;\mathbb E_{\mathbf x_0,\boldsymbol\epsilon}[\|(\gamma_1-1)\boldsymbol\epsilon_\theta(\mathbf x_1,c,1)+\gamma_0\boldsymbol\epsilon_\theta(\mathbf x_1,1))\|_2^2]+C_5,
\end{align}
in which $C_3$, $C_4$, and $C_5$ are constants not involving $\gamma_1$ and $\gamma_0$.
\end{proof}

\section{Pseudo-codes of Lookup Table}\label{sec:supp_exp}

We below propose the pseudo-codes to achieve the lookup table and corresponding guided sampling in \Cref{alg:lookup_table,alg:guided_sampling}.

\begin{algorithm}[!ht]
\caption{Pseudo-code to achieve lookup table of \method in a PyTorch-like style.}
\label{alg:lookup_table}
\begin{lstlisting}[language=python]
def calculate_lookup_table(net, gnet, data_loader, timesteps):
    """Defines the function to maintain the lookup table.

    Args:
        net: Noise prediction model for conditional score function.
        gnet: Noise prediction model for unconditional score function.
        data_loader: Dataloader to calculate score functions.
        timesteps: All timesteps under the given sampling trajectory.

    Returns:
        coeffs: Lookup list under all timesteps and conditions.
    """
    sum1_dict, sum2_dict = dict(), dict()
    # Iterate the dataloader.
    for x, c in data_loader:
        # Iterate for all timesteps.
        sum1s, sum2s = list(), list()
        for nfe_idx, t in enumerate(timesteps):
            # Forward process.
            noise = torch.randn_like(x)
            x_t = alpha_t * x + sigma_t * noise

            # Calculate score functions first.
            eps_cond, eps_uncond = net(x_t, c, t), gnet(x_t, t)

            # Calculate the expectation.
            sum1s.append(eps_cond.mean(dim=0, keepdim=True))
            sum2s.append(eps_uncond.mean(dim=0, keepdim=True))

        # Save the results.
        update_dict(sum1_dict, sum2_dict, c, sum1s, sum2s)

    # Calculate coefficients according to Eq. (34) for all timesteps.
    coeffs = {c: sum1_dict[c] / sum2_dict[c] for c in sum1_dict}

    # Record the mean coefficient for other conditions.
    coeffs.update(avg=sum(coeffs.values()) / len(coeffs)

    return coeffs
\end{lstlisting}
\end{algorithm}

\begin{algorithm}[!ht]
\caption{Pseudo-code for guided sampling by lookup table of \method in a PyTorch-like style.}
\label{alg:guided_sampling}
\begin{lstlisting}[language=python]
def guided_sampler(sampler, net, gnet, gamma_1, noise, c, timesteps, coeffs):
    """Defines the guided sampling with lookup table.
    
    Args:
        sampler: Native sampler without guidance, e.g., DDIM sampler.
        net: Noise prediction model for conditional score function.
        gnet: Noise prediction model for unconditional score function.
        gamma_1: Guidance strength similar to CFG of type `float`.
        noise: Initial random noise to denoise.
        c: Input label.
        timesteps: All timesteps under the given sampling trajectory.
        coeffs: Pre-calculated lookup table.

    Returns:
        x: A batch of samples by guided sampling.
    """
    # Calculate gamma_0.
    if c in coeffs:
        gamma_0s = (1. - gamma_1) * coeffs[c]
    else:
        gamma_0s = (1. - gamma_1) * coeffs['avg']
    # Ensure gamma_0 <= 0 and gamma_1 + gamma_0 >= 1.
    gamma_0s = clamp(gamma_0s, gamma_1)

    # Guided sampling using gamma_1 and gamma_0.
    x = noise
    for t, gamma_0 in zip(timesteps, gamma_0s):
        # Calculate score functions and apply guided sampling.
        eps_cond, eps_uncond = net(x, c, t), gnet(x, t)
        eps = eps_cond * gamma_1 + eps_uncond * gamma_0
        x = sampler(x, eps, t)

    return x
\end{lstlisting}
\end{algorithm}